\documentclass[11pt, a4paper, singlecolumn, gdm]{google}
\usepackage[authoryear, sort&compress, round]{natbib}
\usepackage[english]{babel}
\usepackage{tcolorbox}
\usepackage[ruled,linesnumbered,noend]{algorithm2e}
\SetCommentSty{textrm} 
\SetKwComment{tcp}{$\triangleright$\ }{}

\usepackage{tikz}
\usetikzlibrary{arrows.meta,patterns,positioning}

\newcommand{\eps}{\varepsilon}
\providecommand{\coloneqq}{:=}
\usepackage{caption}
\usepackage{subcaption}
\usepackage{cleveref}
\newtheorem{theorem}{Theorem}
\newtheorem{lemma}[theorem]{Lemma}
\newtheorem{remark}[theorem]{Remark}

\crefname{theorem}{Theorem}{Theorems}
\crefname{lemma}{Lemma}{Lemmas}
\crefname{proposition}{Proposition}{Propositions}
\crefname{assumption}{Assumption}{Assumptions}
\crefname{algorithm}{Algorithm}{Algorithms}
\crefname{remark}{Remark}{Remarks}

\AddToHook{env/lemma/begin}{\crefalias{theorem}{lemma}}
\AddToHook{env/remark/begin}{\crefalias{theorem}{remark}}
\AddToHook{env/proposition/begin}{\crefalias{theorem}{proposition}}
\AddToHook{env/assumption/begin}{\crefalias{theorem}{assumption}}

\uselogo{} 

\def\R{{\mathbb{R}}}

\def\calA{{\mathcal{A}}}

\def\calD{{\mathcal{D}}}

\makeatletter
\newcommand{\corner@glyph}[2]{%
  \begingroup
  \setbox0=\hbox{$\m@th#1\llcorner$}%
  \vcenter{\hbox{%
    \raisebox{\dimexpr\ht0-\height\relax}[\ht0][\dp0]{$\m@th#1#2$}%
  }}%
  \endgroup
}
\newcommand{\lowercorner}{\mathord{\mathpalette\corner@glyph\llcorner}}
\newcommand{\uppercorner}{\mathord{\mathpalette\corner@glyph\urcorner}}
\makeatother

\newbool{showplanbox}
\boolfalse{showplanbox}   

\newtcolorbox{internalplanbox}[1][]{
  colframe = gray!25,
  colback  = gray!10,
  coltitle = gray!20!black,  
  title    = {\textbf{Plan}},
  #1,
}

\NewEnviron{planbox}[1][]{%
  \ifbool{showplanbox}{%
    \begin{internalplanbox}[#1]
      \BODY
    \end{internalplanbox}%
  }{%
  }%
}

\title{Ambiguous Fits: The Rashomon Effect and\\Challenges in Scaling Law Estimation}

\correspondingauthor{bkm28@cam.ac.uk, sfrei@google.com}

\reportnumber{0001} 

\author[1]{Bruno Mlodozeniec$^*$}
\author[2]{Spencer Frei}
\affil[1]{University of Cambridge}
\affil[2]{Google DeepMind}

\begin{abstract}
Estimating neural scaling laws involves an end-to-end pipeline: designing a scaling recipe, training models across scales, evaluating models on a target metric, and deploying an algorithm to fit a scaling law to the resulting data. This final fitting stage is often treated as routine: standard procedures minimize a non-convex fitting loss using local optimization from a grid of initializations, and return the best solution found.
Yet such heuristic searches can fail in two distinct ways: they offer no guarantee of returning a nearly optimal solution, and they obscure whether fundamentally different parameter choices fit the data nearly as well. Indeed, we show that standard scaling law fitting is subject to a Rashomon effect~\citep{breiman2001}: the sublevel sets of near-minimizers correspond to broad parameter regions whose extrapolations diverge substantially at scale. We resolve these issues with a deterministic branch-and-bound algorithm that guarantees nearly optimal fitting loss and exhaustively maps the near-minimizers to the implied range of scaling curves.
\end{abstract}
\begin{document}
\maketitle
\renewcommand{\thefootnote}{\fnsymbol{footnote}}
\footnotetext[1]{Work done while a Student Researcher at Google DeepMind.}
\renewcommand{\thefootnote}{\arabic{footnote}}

\begin{figure}[!htb]
    \centering
    \includegraphics[width=1.0\linewidth]{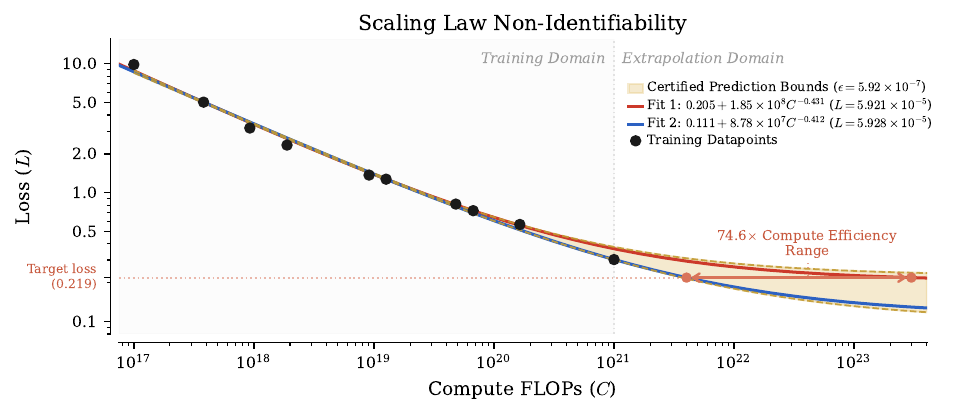}
    \caption{\textbf{The Rashomon Effect in Scaling Law Fitting.} Synthetic scaling ladder data on which two fits with widely different extrapolations achieve nearly identical fitting losses. The yellow band, computed with our certified prediction-bound algorithm (\Cref{sec:certified-prediction-envelopes}), contains the predicted loss of every fit whose fitting loss is within $1\%$ of the optimum. Details are in \Cref{sec:experimental-details}.}
    \label{fig:non-identifiability-scaling-law}
\end{figure}
\section{Introduction}
\label{sec:intro}

\label{sec:scaling-laws-llms}

Neural scaling laws describe empirical power-law relationships between training resources (e.g., compute, data, and parameters) and the resulting model performance (e.g., validation loss)~\citep{kaplan2020scaling,hoffmann2022training}.  These empirical relationships are an important tool in frontier AI model development pipelines: scaling laws end up shaping architecture decisions, compute allocations, and more.   The estimation of a scaling law involves a series of decisions, roughly consisting of four distinct stages:
\begin{enumerate}
    \item \textbf{Design of the scaling recipe}: selecting model architectures, hyperparameters, and data mixtures for models with different compute, parameter, and token budgets.
    \item \textbf{Training}: training models across scales.
    \item \textbf{Evaluation}: measuring trained model performance on a target metric (e.g., the cross-entropy loss on a held-out validation set).
    \item \textbf{Scaling law fitting}: applying a fitting algorithm to the observed dataset of \texttt{(scale, loss)} pairs, where \texttt{scale} can represent FLOPs, parameters, tokens, or other quantities. This returns the parameters $\theta$ of a parametric scaling curve $\texttt{Loss}(\texttt{scale};\theta)$.
\end{enumerate}
While significant effort has been focused on the earlier stages of this pipeline \citep{porian2024resolving,gadre2025gemstones,choshen2024hitchhiker,li2026spend,madaan2024quantifying}, the final fitting stage has largely been treated as routine.  The standard approach in the literature is known as Chinchilla Approach 3~\citep{hoffmann2022training}, which parameterizes the scaling law as a sum of power laws,
\[ L(x; \theta) = E + \sum_{i=1}^k A_i x_i^{-\alpha_i}, \quad \theta = (E, A_1, \alpha_1, \dots, A_k, \alpha_k),\]
and estimates $\theta$ by minimizing a Huber loss over the residual differences between $\log L(x;\theta)$ and observed (log-transformed) losses on the scaling ladder. This minimization is performed by selecting a grid of initialization points, running a local optimization method such as L-BFGS from each point, and choosing the resulting parameters that attain the smallest empirical fitting loss.

There are a number of potential problems with this approach, however. The first is that the fitting loss is non-convex, so local optimization methods are not guaranteed to find near-global minima. Although this approach will return a fit with nearly optimal loss in the limit of an infinitely fine initialization grid, it is unclear how far the fit returned by the default approach is from optimal.

The second issue is that heuristic optimizers return a single parameter fit, obscuring the fact that numerical solvers only guarantee finding an arbitrary element from an $\eps$-sublevel set of near-minimizers. Because scaling law loss landscapes lack strong convexity, this sublevel set is subject to a pronounced Rashomon effect~(\citet{breiman2001}; see~\citet{rudin2024amazing} for a recent perspective): it spans broad parameter regions that fit the training ladder equally well, but make diverging predictions under extrapolation (see Figure~\ref{fig:non-identifiability-scaling-law}). 
Which element of this set is returned depends on incidental choices, such as the initialization grid and the solver settings, which vary considerably across the literature~\citep{li2025misfitting}. Such choices are naturally judged by how well they minimize the fitting loss. Less attention has been paid to the fact that fits with nearly equal loss can extrapolate very differently, so that these choices, even in cases where they barely affect the fitting loss, can substantially change the forecast.
This motivates characterizing the range of predictions made by all fits in the sublevel set, rather than relying on a single point estimate.

In this work, we propose replacing heuristic local search in scaling law fitting with a certified optimization framework.  The foundation of our approach is the observation that for the standard power-law formulation, a coordinate reparameterization makes it straightforward to compute uniform lower and upper bounds on log-residual fitting losses over any parameter box. These loss bounds enable a branch-and-bound algorithm that systematically prunes suboptimal regions of parameter space while refining promising ones. Rather than returning an arbitrary local minimum, this procedure provides two complementary guarantees: it certifies when a candidate fit is within an $\eps$-tolerance of the global optimum, and it exhaustively maps out the full family of near-optimal fits to compute certified prediction envelopes at unobserved scales.

The main contributions of our work can be summarized as follows.

\begin{enumerate}
    \item \textbf{Characterization of the Rashomon effect in scaling law estimation.}  Across synthetic datasets and four real empirical scaling ladders from the literature, we show that scaling law fitting is subject to a pronounced Rashomon effect: the $\eps$-sublevel sets of near-minimizers span broad parameter regions that achieve near-optimal fitting loss on the observed data while yielding widely diverging predictions at larger scales. These findings indicate that forecasts at larger scales should not rely on an arbitrary single fit, but should account for the full family of near-minimizers.
    \item \textbf{Certified optimization and prediction envelopes.} We develop a deterministic branch-and-bound algorithmic framework for power-law scaling law fitting that provides three concrete capabilities: 
    \begin{enumerate}[label=(\alph*),itemsep=0.2em] 
    \item \textit{Certified nearly optimal fits:} For any user-specified tolerance $\eps>0$, our algorithm provably returns parameters $\theta$ whose fitting loss is within $
    \eps$ of the global minimum over a declared compact domain. 
    \item \textit{Certified sublevel set cover:} To identify the full family of scaling laws that achieve near-optimal fits, we provide an algorithm that computes an arbitrarily tight cover of the $\eps$-sublevel set of near-minimizers of the fitting loss.
    \item \textit{Certified prediction envelopes:} For any unobserved target scale, our algorithm computes provable lower and upper bounds on the predicted loss across all parameter configurations in the $\eps$-sublevel set, guaranteeing that no near-optimal fit can forecast a value outside this envelope. 
    \end{enumerate}
\end{enumerate}

\section{Background}
\label{sec:formulation}

We begin by formalizing the scaling law fitting pipeline. We introduce the power-law hypothesis class, describe the empirical loss objectives used to fit observed scaling data, and review the standard heuristic optimization procedures used in practice.

\subsection{The Scaling Law Estimation Problem}
\label{sec:estimation-problem}

Below, we first introduce the various components of the scaling law estimation problem, including the parametric form, fitting loss, and the fitting algorithm.

\textbf{Parametric form.} In scaling law estimation, we are typically interested in identifying plausible mappings from \textit{scale} $x \in \R^k$ (e.g., $x=F$, where $F$ denotes the training FLOPs, or $x=(N, D)$, where $N$ is the number of model parameters and $D$ is the number of training tokens~\citep{kaplan2020scaling}) to observations of a final validation loss $y \in \R$. We typically restrict our attention to a power-law \textbf{parametric family} (hypothesis class) $\mathcal{H}$ of functions, such as
\begin{align}
L(x; \theta) = E + \sum_{i=1}^k A_i \cdot x_i^{-\alpha_i}, \quad \theta = (E, A_1, \alpha_1, \dots, A_k, \alpha_k). \label{eq:chinchilla-power-law-form}
\end{align}
The above form, given in \citep{hoffmann2022training}, is arguably the most popular hypothesis class for scaling law fitting, and we primarily focus on this class in this work. However, other common choices exist.\footnote{E.g. the $L((N, D);\theta) = \left(A_N /N^{\alpha_N / \alpha_D} + A_D / D \right)^{\alpha_D}$ form in \citep{kaplan2020scaling}, or the form with the coupling exponent $L(x; \theta) = E + \left(\sum_{i=1}^k A_i \cdot x_i^{-\alpha_i}\right)^\gamma, \theta = (E, A_1, \alpha_1, \dots, A_k, \alpha_k,\gamma)$ in \citep{busbridge2025distillationscalinglaws,bethune2026designspacetrimodalmasked,videau2026skalingchinchillasexponentsmeet,ramapuram2026scalingpropertiescontinuousdiffusion}.}

\textbf{Training data.} The goal is typically to find members of the family $\mathcal{H}$ that fit a \textbf{training dataset}, i.e., a sequence of $N_{\mathrm{obs}}$ collected observations $\mathcal{D}=\left((x^{(n)}, y^{(n)})\right)_{n=1}^{N_{\mathrm{obs}}}$. Let $\mathfrak{D}=\cup_{N_{\mathrm{obs}}=1}^\infty \left(\R^k_{>0} \times \R_{>0}\right)^{N_{\mathrm{obs}}}$ denote the space of possible datasets. We assume that the scales and targets are strictly positive: $x^{(n)}, y^{(n)} \succ 0$.

\textbf{Fitting algorithm.} The ``best'' fit from a parametric family is identified by a \textbf{fitting algorithm} $\mathcal{A}: \mathfrak{D} \to \R^{2k + 1}$ that takes a dataset $\mathcal{D}$ and returns the ``best'' parameters. The fitting algorithm typically approximately minimizes a \textbf{fitting loss} $L_\texttt{fit}: \R^{2k + 1} \times \mathfrak{D} \to \R$. For example,
\begin{align*}
    L_\texttt{fit}(\theta; \mathcal{D}) = \frac{1}{N_{\mathrm{obs}}} \sum_{n=1}^{N_{\mathrm{obs}}} \ell(L(x^{(n)}; \theta), y^{(n)}),
\end{align*}
where the \textbf{residual loss} $\ell: \R \times \R \to \R$ might, for instance, be the $L_2$ loss $\ell(\hat{y}, y)= (\hat y - y)^2$, or the Huber loss in log-space $\ell(\hat{y}, y)= \operatorname{Huber}_{10^{-3}} (\log \hat{y} - \log y)$ (where $\operatorname{Huber}_\delta(z)=\tfrac12 z^2$ if $\lvert z\rvert\le\delta$ and $\operatorname{Huber}_\delta(z)=\delta(\lvert z\rvert-\tfrac{\delta}{2})$ otherwise is the Huber loss~\citep{huber1964robust} with parameter $\delta$). The latter is the most commonly used residual loss and is used, for example, in Chinchilla Approach~3~\citep{hoffmann2022training}. Given a fitting loss, the choice of an \textbf{optimizer} fully specifies a fitting algorithm; the optimizer dictates how well the fitting loss is optimized and which parameter configuration is returned when several attain the same fitting loss.

Ultimately, the output of the fitting algorithm, $\theta^\ast := \calA(\calD)$, will be used for extrapolation: to predict the validation loss as $L(x^\ast; \theta^\ast)$ at a larger, previously unobserved scale $x^\ast$.

\subsection{Chinchilla Approach 3}
\label{sec:chinchilla-fitting}

In the remainder of this work, we focus primarily on the widely used Chinchilla approach \citep{hoffmann2022training}. For the parametric family \Cref{eq:chinchilla-power-law-form} introduced above, this approach minimizes
\begin{align*}
    L_{\texttt{fit}}(\theta;\mathcal{D})
    =
    \frac{1}{N_{\mathrm{obs}}}\sum_{n=1}^{N_{\mathrm{obs}}}
    \operatorname{Huber}_{10^{-3}}
    \left(
        \log L(x^{(n)};\theta)-\log y^{(n)}
    \right),
\end{align*}
which we refer to as the \textbf{log-residual Huber loss}. The original Chinchilla setting has $k=2$ and hence five parameters, but the same objective applies without modification to any input dimensionality $k$, in which case $\theta \in \R^{2k+1}$.

The accompanying optimizer uses multi-start local optimization. Writing $e=\log E$ and $a_i=\log A_i$, the original $k=2$ procedure initializes from the Cartesian grid $e\in\{-1,-0.5,\ldots,1\}, a_i\in\{0,5,\ldots,25\}, \alpha_i\in\{0,0.5,\ldots,2\}$, runs L-BFGS from every grid point, and returns the optimized parameters with the lowest fitting loss across the grid. Unless stated otherwise, we follow common practice in the literature~\citep{busbridge2025distillationscalinglaws} and use \texttt{scipy.optimize.minimize} with \texttt{method="L-BFGS-B"} and otherwise default settings in all experiments~\citep{2020SciPy}.
For general $k$, we use the same boundaries for each corresponding parameter: $e\in[-1,1]$, $a_i\in[0,25]$, and $\alpha_i\in[0,2]$ for every $i$.
We note that these boundaries and the specific grid density are entirely heuristic~\citep{hoffmann2022training}.
In some experiments, we control the \textbf{grid density} $K$ as a proxy for `optimization effort'. We retain the same ranges and place $K$ values at the midpoints of $K$ equal subintervals of each range, yielding $K^{2k+1}$ initializations.

\section{Ambiguous Fits: The Rashomon Effect in Scaling Law Estimation}
\label{sec:certified}
Although power-law scaling curves contain only a handful of parameters, fitting them to datasets that often contain only a few dozen data points presents subtle optimization and statistical challenges. Standard multi-start local solvers leave two critical questions unaddressed: whether the returned parameters are close to globally optimal, and whether alternative near-optimal fits yield fundamentally different predictions at scale. These are the questions we investigate in this section.

\paragraph{Heuristic fitting does not certify how well it optimizes.}
The widely used Chinchilla fitting procedure applies a local optimizer from a hand-chosen grid of initializations (see \Cref{sec:chinchilla-fitting}).
The procedure is deterministic once that grid is fixed, but the returned fit depends on which initializations were used and on the hyperparameters of the local optimizer.
It returns a fit and its fitting loss, but no lower bound on the minimum fitting loss, and hence no indication of how suboptimal the returned fit is.
To see how much this can matter for extrapolation, in \Cref{fig:viable-extrap-fans-all} we show the Chinchilla fit together with other fits attaining the same or lower fitting loss on scaling ladder datasets from the literature~\citep{apple2025revisiting,marin2026delphi,besiroglu2024chinchilla,dubey2024llama}. All of these fits are at least as good as the Chinchilla fit by the fitting criterion, yet their predictions vary substantially.
Concretely, for each ladder we fit the five-parameter law $L(N,D)=E+A_1N^{-\alpha_1}+A_2D^{-\alpha_2}$ to all observations using Chinchilla Approach~3 with $K^5$ initializations, $K\in\{2,5,8\}$, and record the attained loss $L_{\mathrm{chin},K}$. We then search the sublevel set $\{\theta:L_{\texttt{fit}}(\theta)\le L_{\mathrm{chin},K}\}$ for the fits whose predictions at $300\times$ the largest training compute are highest and lowest, and plot their scaling curves along the slice $D=20N$ (\Cref{sec:experimental-details}).
The envelope at density $K$ is determined by the attained loss $L_{\mathrm{chin},K}$: a lower loss gives a smaller sublevel set, and hence a narrower envelope.
Denser grids usually attain lower losses, and in \Cref{fig:viable-extrap-fans-all} the envelopes generally narrow as $K$ increases.
This is not guaranteed, however. The grids for different $K$ are not nested (each places $K$ points at the cell centers of each coordinate range, so the $2^5$ grid is not a subset of the $5^5$ grid), and the best of the local optimizations depends on which basins the initializations fall into; on Llama~3, for example, the $2^5$ grid attains a lower loss than the $5^5$ grid.
More fundamentally, without a lower bound on the minimum fitting loss, we cannot tell how far $L_{\mathrm{chin},K}$ is from the global minimum, and hence how much of the envelope reflects imperfect optimization rather than ambiguity inherent in the data.  Using the certified bounds of~\Cref{sec:main-result}, we find that the $5^5$ fits exceed the minimum fitting loss by between 2\% and 180\% across the four ladders.  The envelopes in Figure~\ref{fig:viable-extrap-fans-all} therefore reflect both imperfect optimization and the ambiguity we turn to next.

\paragraph{Near-optimal fits can still extrapolate very differently.}
A natural expectation is that this ambiguity disappears with better optimization: once the fitting loss is close to its minimum, which near-minimizer is returned should have little effect on the extrapolation.
This need not be the case. Scaling law estimation can exhibit the Rashomon effect~\citep{breiman2001}: the $\eps$-sublevel set can contain near-minimizers with virtually identical fitting loss on observed data but substantially different scaling law curves. In Figure~\ref{fig:non-identifiability-scaling-law}, we illustrate this on a synthetic scaling ladder: two power-law fits whose fitting losses are certified to be within $1\%$ of the optimum predict a $74.6\times$ difference in the compute needed to reach a given loss at $300\times$ the training range.
Returning a single fit obscures this family of near-optimal alternatives, presenting one of them as though it were uniquely determined by the data.

\begin{figure}[!t]
    \centering
    \includegraphics[width=\linewidth]{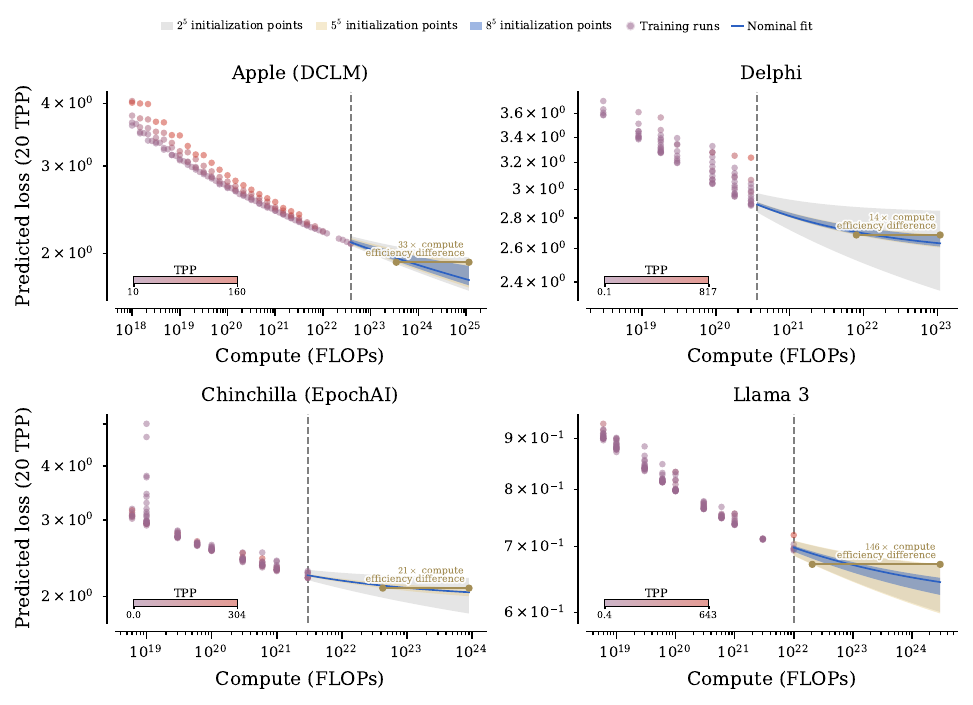}
    \caption{\textbf{The Rashomon Effect and Sublevel Prediction Envelopes Across Scaling Ladders.} 
    We run Chinchilla Approach~3 on four different scaling ladders taken from the literature (see \Cref{app:scaling-ladder-datasets}) with different initialization grid densities (following the standard heuristic parameter ranges $\log E \in [-1, 1]$, $\log A_1, \log A_2 \in [0, 25]$, and $\alpha_1, \alpha_2 \in [0, 2]$), where a grid with $5^5$ initialization points is closest to the \citet{hoffmann2022training} default.
    We record the attained fitting loss $L_{\mathrm{chin}, K}$ at each density $K$, search for alternative parameter configurations in the sublevel set achieving $L_{\texttt{fit}}(\theta) \le L_{\mathrm{chin},K}$, and plot the resulting prediction envelopes along a canonical $20$ tokens-per-parameter ($D = 20N$) slice.  The extreme fits are found by local search (\Cref{sec:experimental-details}), so the true envelopes can only be wider.
    With the standard Chinchilla approach at a $5^5$ grid initialization (gold band), we observe a pronounced Rashomon effect: fits in the sublevel set that match the training data just as well as the returned fit yield a compute efficiency difference between $14\times$ and $146\times$ for a target loss. Horizontal segments indicate this compute efficiency range.  While increasing grid density $K$ generally shrinks the sublevel envelope  (\Cref{sec:certified}), heuristic optimization provides no certificate of how close the returned fit is to global optimality or how wide the true envelope remains.}
    \label{fig:viable-extrap-fans-all}
\end{figure}

\section{Certified Fitting and Sublevel Prediction Envelopes}
\label{sec:main-result}

Scaling law forecasts can determine extremely expensive training decisions in a typical AI research pipeline.
Relying on an unverified single fit leaves practitioners with no guarantee of how close that fit is to the global optimum, or what alternative scaling curves fit the observed data nearly as well.
In this section, we introduce a branch-and-bound approach that resolves these issues by providing three methods with three concrete guarantees: (a) a method that returns a candidate fit certified to be within an $\epsilon$-tolerance of the global minimum, (b) a method for constructing an arbitrarily tight box cover of the entire $\epsilon$-sublevel set of near-minimizers, and (c) a method for computing tight lower and upper bounds on the prediction envelopes across all viable fits.

Our core results are given below in \Cref{thm:main}. All guarantees are relative to a declared compact search domain $\Theta$, of the form made precise in \Cref{sec:setup-monotonicity}.\footnote{In practice, we always choose this domain to encompass the domain for the Chinchilla Approach~3; i.e. we always make the search range more conservative.}

\newpage
\begin{theorem}
\label{thm:main}
Fix a scaling law dataset $\mathcal{D} = \left((x^{(n)}, y^{(n)})\right)_{n=1}^{N_{\mathrm{obs}}}$, with $x^{(n)} \in \R_{>0}^k, y^{(n)} \in \R_{>0}$, and a compact search domain $\Theta$ of the form specified in \Cref{eq:original-search-domain}.
Define the $\epsilon$-sublevel set $\mathcal{S}_\epsilon = \{\theta \in \Theta: L_\texttt{fit} (\theta) - \min_{\theta' \in \Theta} L_\texttt{fit} (\theta') \leq \epsilon\}$. For the scaling law $L(x; \theta) = E + \sum_{i=1}^k A_i x_i^{-\alpha_i}$ with log-residual Huber fitting loss, the following branch-and-bound algorithms are provided:

\begin{enumerate}[label=(\alph*)]
    \item \textbf{Certified fitting.} For any $\epsilon > 0$, \Cref{alg:certified-fit} returns a fit whose original parameters $\hat{\theta}$ satisfy $\hat{\theta} \in \mathcal{S}_\epsilon$ (i.e., $L_{\texttt{fit}}(\hat{\theta}) - \min_{\theta \in \Theta} L_\texttt{fit} (\theta) \leq \epsilon$) in a finite number of iterations.

    \item \textbf{Certified sublevel set cover.} For any $\epsilon > 0$, \Cref{alg:viable-set-cover} returns a finite representation of a region $\mathcal C\subseteq\Theta$ satisfying $\mathcal S_{\epsilon/2}\subseteq\mathcal C\subseteq\mathcal S_\epsilon$.
    
    \item \textbf{Certified sublevel prediction envelope.} For a threshold $\epsilon>0$, a slack parameter $\eta>0$, and a test point $x^\ast$, define the set of ``viable'' predictions at $x^\ast$ as $\mathcal{L}_{x^\ast, \epsilon}=\{L(x^\ast; \theta): \theta \in \mathcal{S}_\epsilon\}$. \Cref{alg:certified-prediction-bound} returns bounds $[\underline{L}_{x^\ast}, \overline{L}_{x^\ast}]$ such that
    \begin{align*}
    \inf \mathcal{L}_{x^\ast, \epsilon + \eta} \leq \underline{L}_{x^\ast} \leq \inf \mathcal{L}_{x^\ast, \epsilon} , && \text{and} && \sup \mathcal{L}_{x^\ast, \epsilon} \leq \overline{L}_{x^\ast} \leq \sup \mathcal{L}_{x^\ast, \epsilon+\eta}.\\
    \end{align*}
    In other words, the returned interval contains the prediction range at $x^\ast$ for parameters in $\mathcal{S}_\epsilon$ and is contained within the prediction range for parameters in $\mathcal{S}_{\epsilon + \eta}$.
\end{enumerate}
\end{theorem}
The above theorem says that
\begin{enumerate}[label=(\alph*)]
    \item \Cref{alg:certified-fit} provably finds a parameter configuration within $\eps$ of the global minimum on the declared search domain, eliminating the optimizer suboptimality gap that heuristic methods cannot resolve.
    \item \Cref{alg:viable-set-cover} tightly bounds the entire $\epsilon$-sublevel set of near-minimizers, capturing which fits are consistent with small fitting loss.
    \item \Cref{alg:certified-prediction-bound} tightly bounds how much predictions can vary at an unobserved scale across all near-optimal fits, quantifying the extrapolation uncertainty arising from the Rashomon effect.
\end{enumerate}

In the rest of this section, we describe the algorithms and explain why they achieve these guarantees.
We first explain the general idea behind branch-and-bound algorithms for global optimization without reference to the scaling-law formulas.
We then show how a monotone reparameterization provides the bounds needed by a branch-and-bound algorithm and describe the fitting, prediction, and solution-set searches and their guarantees.
Complete details and proofs are given in \Cref{app:proofs}.


\subsection{Branch and Bound Global Optimization}
\label{sec:branch-and-bound}


The core idea behind branch-and-bound global optimization~\citep{horst1996global,lawler1966branch,floudas2000deterministic,hansen2003global} is to use inexpensive lower bounds on the fitting loss over parameter regions. Using these lower bounds we can either discard whole regions of the search space --- by showing that their lower bounds are worse than the best fit found so far --- or certify candidate solutions as nearly globally optimal --- by showing that the candidate solution has loss within $\epsilon$ of the best possible lower bound.
A region whose loss lower bound is already worse than the best fit we have found cannot contain the optimum, so we can discard it.
A region that might contain a better fit can be split into smaller regions to obtain a better lower bound. We can repeat these steps, while refining the best candidate solution, until the lower bound (across all regions) is within $\epsilon$ of the loss achieved by the best candidate solution.
Once this holds, since we know the best fitting loss must be higher than the lower bound, we can conclude that the candidate solution is within $\epsilon$ of achieving the best possible fitting loss.

\paragraph{Notation.} For a parameter box $B=\prod_{j=1}^d[a_j,b_j]\subseteq\R^d$, denote its componentwise lower and upper corners by $B^{\lowercorner}=(a_1,\ldots,a_d)$ and $B^{\uppercorner}=(b_1,\ldots,b_d)$, respectively. Their $j$-th coordinates are $B^{\lowercorner}_j=a_j$ and $B^{\uppercorner}_j=b_j$, and we write
\[
    B=[B^{\lowercorner},B^{\uppercorner}]
    :=\{\theta\in\R^d:B^{\lowercorner}_j\leq\theta_j\leq B^{\uppercorner}_j
    \text{ for every }j\}.
\]
In this subsection, $\theta$ denotes a parameter vector in the coordinates used for the search (temporarily overloading the previous notation). Branch and bound algorithms assume that we have access to some bound functions $\underline{L}_{\texttt{fit}}(B)$ and $\overline{L}_{\texttt{fit}}(B)$, that are easy to compute, for which we get:
$
    \underline{L}_{\texttt{fit}}(B)
    \leq L_{\texttt{fit}}(\theta;\mathcal D)
    \leq \overline{L}_{\texttt{fit}}(B)
$ for every $\theta\in B$.
We also maintain a candidate solution $\theta_B$ within each box $B$, satisfying $\underline{L}_{\texttt{fit}}(B)
    \leq L_{\texttt{fit}}(\theta_B;\mathcal D)
    \leq \overline{L}_{\texttt{fit}}(B)$.
Starting from an initial search domain $B_\texttt{global}$, we repeatedly split boxes and discard those whose lower bounds exceed the loss of the best candidate solution found so far.
Under some minor continuity assumptions, for appropriately chosen bound functions, the lower and upper bound functions $\underline{L}_{\texttt{fit}}(B), \overline{L}_{\texttt{fit}}(B)$ should converge to one another as we keep splitting, eventually guaranteeing that the variation in the fitting loss in each box is smaller than $\epsilon$, in which case the best candidate solution is within $\epsilon$ of the minimum.  The algorithm can terminate sooner than that, by returning a candidate solution once it is within $\epsilon$ of the lowest fitting loss lower bound across all regions.


\subsection{The Algorithms}
\label{sec:setup-monotonicity}

The only model-specific ingredients are what loss bounds we use for each box $B$.
Below, we show that a change of coordinates provides simple lower and upper bounds for the
standard scaling law \Cref{eq:chinchilla-power-law-form}.

\paragraph{Monotonic reparameterization.}
Recall the parametric power law family $ L(x;\theta)=E+\sum_{i=1}^k A_i x_i^{-\alpha_i}$.

For each input coordinate, let $x_i^{\min}=\min_n x_i^{(n)}$ be the coordinatewise minimum and define
\begin{equation}
    s_i(x)=\log\frac{x_i}{x_i^{\min}},\qquad
    C_i=A_i(x_i^{\min})^{-\alpha_i},\qquad
    r_i=e^{-\alpha_i}.
    \label{eq:certified-reparameterization}
\end{equation}
Writing $u=(E,C_1,\ldots,C_k,r_1,\ldots,r_k)$ gives the equivalent form
\begin{equation}
    L(x;u)=E+\sum_{i=1}^k C_i r_i^{s_i(x)}.
    \label{eq:monotone-scaling-law}
\end{equation}
Note that this is simply a bijective reparameterization of the original problem, which we denote by $u=\phi(\theta)$. Its inverse leaves $E$ unchanged and recovers $\alpha_i=-\log r_i$ and $A_i=C_i(x_i^{\min})^{\alpha_i}$.

The three certified algorithms below search a fixed compact hyperrectangle
\begin{equation}
\mathcal U
=
[\underline E,\overline E]
\times
\prod_{i=1}^k[\underline C_i,\overline C_i]
\times
\prod_{i=1}^k[\underline r_i,\overline r_i],
\label{eq:certified-search-domain}
\end{equation}
where $0<\underline E\leq\overline E<\infty$, $0<\underline C_i\leq\overline C_i<\infty$, and $0<\underline r_i\leq\overline r_i\leq1$ for all $i$. These will be the standing assumptions for all lemmas and guarantees below.\footnote{Our implementation also permits zero lower bounds, but those boundary cases are not covered by the proofs given here. We believe that the proofs extend to cover these cases, but we stick with our assumptions for simplicity.}
In particular, every prediction is at least $\underline E>0$, so log-predictions are well-defined.
The corresponding domain in the original parameters is
\begin{equation}
    \Theta=\phi^{-1}(\mathcal U).
    \label{eq:original-search-domain}
\end{equation}
Global guarantees are always relative to $\mathcal U$, or equivalently to $\Theta$.\footnote{Broad bounds in the original coordinates can be converted into a conservative enclosing box of this form. For example, \(A_i\in[A_i^-,A_i^+]\) and \(\alpha_i\in[\alpha_i^-,\alpha_i^+]\) imply \(r_i\in[e^{-\alpha_i^+},e^{-\alpha_i^-}]\) and $
C_i\in
\left[
A_i^-\min_{\alpha\in[\alpha_i^-,\alpha_i^+]}(x_i^{\min})^{-\alpha},
\;A_i^+\max_{\alpha\in[\alpha_i^-,\alpha_i^+]}(x_i^{\min})^{-\alpha}
\right].
$
The exact transformed image need not be a box, so this enclosure can enlarge the original search region. The guarantees then apply to the enlarged domain $\Theta=\phi^{-1}(\mathcal U)$.
}
For simplicity, we will write $L_{\texttt{fit}}(u)$ for $L_{\texttt{fit}}(\phi^{-1}(u))$, and similarly for predictions, where apparent from context. The change of coordinates preserves both the fitting loss and its minimum over the declared domain.

This reparameterization gives the scaling law in \Cref{eq:monotone-scaling-law} a useful monotonicity property.
It can be readily seen from \Cref{eq:monotone-scaling-law} that, since $s_i(x^{(n)})\geq0$, increasing any coordinate of $u$ can only increase the
prediction on a training point. Consequently, the two opposite corners of a parameter box bound every prediction made inside it: for any box $B\subseteq\mathcal U$, we have that $L(x^{(n)};B^{\lowercorner})\leq L(x^{(n)};u) \leq L(x^{(n)};B^{\uppercorner})$ for all $u\in B$.

\begin{lemma}[Monotone prediction map and corner enclosure]
\label{lem:monotonicity}
Let
\(
u=(E,C_1,\ldots,C_k,r_1,\ldots,r_k)
\)
be the reparameterized coordinates defined in \Cref{eq:certified-reparameterization}, so that predictions take the
form in \Cref{eq:monotone-scaling-law}. For every training input
\(x^{(n)}\) in \(\mathcal D\), the maps $u\mapsto L(x^{(n)};u)$ and $u\mapsto \log L(x^{(n)};u)$ are coordinatewise nondecreasing on \(\mathcal U\). Consequently, for every parameter box \(B=[B^{\lowercorner},B^{\uppercorner}]\subseteq\mathcal U\), \begin{equation}
    \log L(x^{(n)};B^{\lowercorner})
    \leq
    \log L(x^{(n)};u)
    \leq
    \log L(x^{(n)};B^{\uppercorner}) \qquad \forall u \in B.
    \label{eq:training-corner-enclosure}
\end{equation}

The same conclusion holds for any input \(x\) whenever \(x_i\geq x_i^{\min}\) for every coordinate \(i\).
For an arbitrary $x$, the prediction is still bounded by appropriate corners of \(B\), with the \(r_i\)-corner reversed whenever \(x_i<x_i^{\min}\); see \Cref{app:arbitrary-target-prediction-bounds}.
\end{lemma}

\Cref{lem:monotonicity} follows directly from \Cref{eq:certified-reparameterization}: \(s_i(x^{(n)})\geq0\), and hence every term \(C_i r_i^{s_i(x^{(n)})}\) in \Cref{eq:monotone-scaling-law} is nondecreasing in both \(C_i\) and \(r_i\). Adding \(E\) preserves the monotonicity.

\paragraph{Monotonicity gives simple fitting loss bounds.}   
Let $B=[B^{\lowercorner},B^{\uppercorner}]$.  On input $x^{(n)}$ in the dataset, the ``corner'' predictions define the interval
\begin{equation}
    p_n^-(B)=\log L(x^{(n)};B^{\lowercorner}),\qquad
    p_n^+(B)=\log L(x^{(n)};B^{\uppercorner}),\qquad
    I_n(B)=[p_n^-(B),p_n^+(B)].
\end{equation}
By \Cref{lem:monotonicity}, all predictions made with parameters $u \in B$ at $x^{(n)}$ must lie within $I_n(B)$.
Let $t_n=\log y^{(n)}$ and define the \emph{unavoidable residual} $d_n(B)=\operatorname{dist}(t_n,I_n(B))$ (where $\operatorname{dist}(t,[a,b])=\max\{a-t,0,t-b\}$): the smallest error observation $n$ could possibly attain anywhere in the box.
Write $\operatorname{Huber}_\delta$ for the Huber loss used by the canonical log-residual objective.  In the transformed coordinates, the same fitting loss is
\begin{equation}
    L_{\texttt{fit}}(u)
    =\frac{1}{N_{\mathrm{obs}}}\sum_{n=1}^{N_{\mathrm{obs}}}
      \operatorname{Huber}_\delta\bigl(\log L(x^{(n)};u)-t_n\bigr).
    \label{eq:fitting-loss-reparameterized}
\end{equation}
If $t_n$ lies inside $I_n(B)$, some prediction in the box could match that observation exactly. If it lies outside, the distance to the interval is an error that \emph{every} fit in the box must incur. Combining the corresponding unavoidable losses gives the box-loss lower bound
\begin{align}
    \underline{L}_{\texttt{fit}}(B)
    &=\frac{1}{N_{\mathrm{obs}}}\sum_{n=1}^{N_{\mathrm{obs}}} \operatorname{Huber}_\delta\!\left(d_n(B)\right).
      \label{eq:box-loss-lower-bound}
\end{align}
This lower bound can be conservative: the best residual for one observation and the best residual for another need not come from the same parameters. However, as we refine the boxes, the bound becomes tighter and tighter.

We can also define a corresponding \emph{uniform box-loss upper bound}. Since the Huber loss is convex, its maximum over the residual interval \([p_n^-(B)-t_n,p_n^+(B)-t_n]\) is attained at an endpoint. Therefore,
\begin{equation}
    \overline L_{\texttt{fit}}(B)
    =
    \frac{1}{N_{\mathrm{obs}}}\sum_{n=1}^{N_{\mathrm{obs}}}
    \max\left\{
        \operatorname{Huber}_{\delta}
        \bigl(p_n^-(B)-t_n\bigr),
        \operatorname{Huber}_{\delta}
        \bigl(p_n^+(B)-t_n\bigr)
    \right\}.
    \label{eq:box-loss-upper-bound}
\end{equation}
Thus, we can get that for every $u\in B$,
\begin{equation}
    \underline{L}_{\texttt{fit}}(B)\leq L_{\texttt{fit}}(u)\leq\overline{L}_{\texttt{fit}}(B).
    \label{eq:box-loss-main-sandwich}
\end{equation}
The gap in these bounds converges to zero with the box diameter: 
as every side length of \(B\) tends to zero, $\overline L_{\texttt{fit}}(B) - \underline L_{\texttt{fit}}(B) \rightarrow 0$.
See \Cref{lem:consistent-loss-enclosure} for details.

We also note that nothing in the above formulation is specific to the Huber loss; substituting most other residual losses is straightforward. The approach can also be made to generalize beyond this exact power-law family. The monotone reparameterization can be adapted to fit many types of scaling laws (or alternatively, other approaches to obtain a cheap lower and upper bound could be employed).


\paragraph{Branch-and-bound scaling-law fitting algorithm.}
Having defined how we can obtain upper and lower bounds $\underline{L}_{\texttt{fit}}(B),\overline{L}_{\texttt{fit}}(B)$ for each region $B$ in the reparameterized space, we can give the complete search algorithm in \Cref{alg:certified-fit}.
The algorithm keeps track of the best loss achieved at a tested point and the loss lower bound of each region still in play.
The difference between the lowest lower bound and the best loss achieved so far tells us how much an undiscovered fit could still improve the answer.
There are two subroutines in \Cref{alg:certified-fit} that we have yet to describe. The first one is $\textsc{SampleCandidate}(B)$, which chooses a point within $B$ to evaluate as a candidate point (for example, its center). In our implementation, since we are already evaluating the exact predictions at the upper and lower corners of $B$ and can compute the fitting loss very cheaply, we use both of these points as candidates for each region, returning the one with the lower fitting loss.
The second one is $\textsc{BranchRegion}(B)$, which returns two disjoint boxes that partition $B$.
In \Cref{alg:certified-fit} and the prediction-bound searches, $\textsc{BranchRegion}(B)$ bisects $B$ at the midpoint of its widest side, $\arg\max_j(B^{\uppercorner}_j-B^{\lowercorner}_j)$; in \Cref{alg:viable-set-cover}, each side width is first divided by the corresponding width of $\mathcal U$.\footnote{
More efficient splitting rules exist, and we expect the choice of rule to be one of the main levers for reducing runtime in practice.
}



\begin{algorithm}[t]
\SetAlgoLined
\KwIn{Search domain hyperrectangle $\mathcal{U}$; desired optimality gap $\epsilon$; scaling law dataset $\calD$}
\KwOut{Best tested fit $\hat{u}$ satisfying $L_{\texttt{fit}}(\hat{u}) \leq \inf_{u \in \mathcal{U}} L_{\texttt{fit}}(u) + \epsilon$; tightest estimate of the loss lower bound.}

$\hat{u} \gets \textsc{SampleCandidate}(\mathcal{U})$ \tcp*{Initialize candidate parameters}
Initialize a min-priority queue $\mathcal{Q}$\;
$\mathcal{Q}.\textsc{Push}(\mathcal{U}, \text{priority}=\underline{L}_{\texttt{fit}}(\mathcal{U}))$ \tcp*{Active regions ordered by loss lower bound}

\While{$\mathcal{Q}$ is not empty}{
    $\underline{L} \gets \mathcal{Q}.\textsc{PeekMinPriority}()$ \tcp*{Lowest possible loss among active regions}
    \If(\tcp*[f]{Current best fit is guaranteed to be $\epsilon$-optimal}){$L_{\texttt{fit}}(\hat{u}) \leq \underline{L} + \epsilon$}{\KwRet $\hat{u},\min\{\underline{L},\ L_{\texttt{fit}}(\hat u)\} $\; }

    $B \gets \mathcal{Q}.\textsc{PopMin}()$ \tcp*{Extract region with the smallest lower bound}

    \ForEach(\tcp*[f]{Bisect the widest side of $B$}){child box $C \in \textsc{BranchRegion}(B)$}{ 
        $u_C \gets \textsc{SampleCandidate}(C)$\;
        \If(\tcp*[f]{Found a better candidate fit}){$L_{\texttt{fit}}(u_C) < L_{\texttt{fit}}(\hat{u})$}{ 
            $\hat{u} \gets u_C$\;
        }

        Compute child lower bound $\underline{L}_{\texttt{fit}}(C)$\;
        \If(\tcp*[f]{Discard regions that cannot improve the best fit}){$\underline{L}_{\texttt{fit}}(C) < L_{\texttt{fit}}(\hat{u})$}{ 
            $\mathcal{Q}.\textsc{Push}(C, \text{priority}=\underline{L}_{\texttt{fit}}(C))$\;
        }
    }
}
\KwRet $\hat{u}, L_\texttt{fit}(\hat{u})$\;
\caption{Certified fitting search}
\label{alg:certified-fit}
\end{algorithm}

We illustrate one step of the algorithm on a toy problem in \Cref{fig:certified-bb-subdivision-step}. The figure shows how splitting a box in parameter space refines the prediction and fitting-loss bounds.

\begin{figure*}[t]
    \centering
    \includegraphics[width=\linewidth]{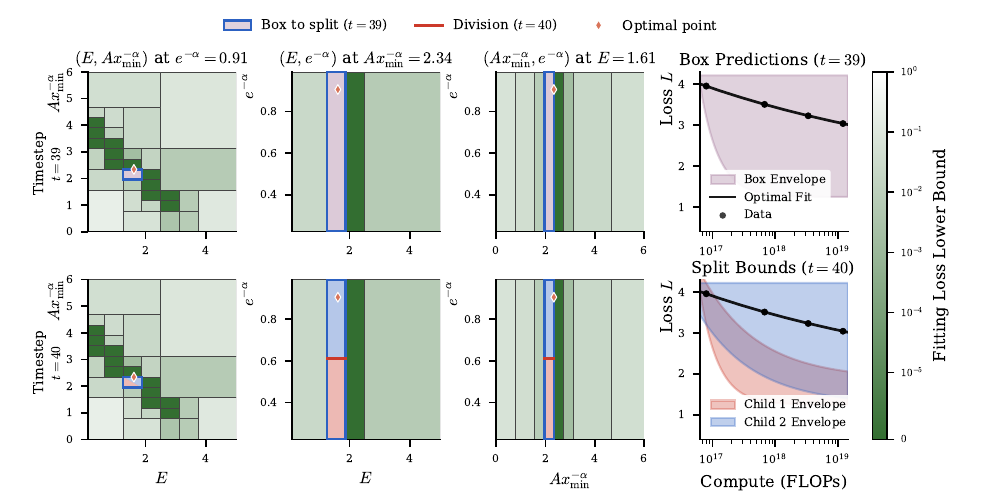}
        \caption{\textbf{Visualization of a branch-and-bound step.} 
    This figure illustrates a single iteration of the certified search \Cref{alg:certified-fit} in the monotonic coordinate space $(E, A x_{\min}^{-\alpha}, e^{-\alpha})$. 
    \textbf{(Top Row)} A parent bounding box with the lowest loss lower bound is selected for splitting. Background heatmaps show the guaranteed loss lower bound at this step. \textbf{(Top Row Right)} The corresponding scaling law predictions for parameters inside the box to be split.
    \textbf{(Bottom Row)} The algorithm branches the box along the exponent dimension $e^{-\alpha}$. This produces two new boxes: Child 1 and Child 2. 
    \textbf{(Bottom Row Right)} The corresponding scaling law predictions for parameters inside the two new child boxes. The algorithm can discard Child 1 because its fitting-loss lower bound over the training grid strictly exceeds the best loss found so far, recursively narrowing the search space around a true optimum.}
    \label{fig:certified-bb-subdivision-step}
\end{figure*}

\paragraph{Measuring suboptimality.} \Cref{alg:certified-fit} bounds not only the loss of the returned incumbent, but also how much better any undiscovered fit could be.
Concretely, for the returned incumbent $\hat{u}$ and tightest fitting loss lower bound $\underline{L}$, $L_\texttt{fit}(\hat{u})$ upper-bounds the true minimum loss, and $\underline{L}$ lower-bounds it, to within $\epsilon$ tightness. By running with a small $\epsilon$, we can use these bounds to estimate the ``suboptimality'' of fits returned from any other fitting algorithm (e.g., Chinchilla). Concretely, a fit $u$ returned by another algorithm is at least $\max\{0,\,L_\texttt{fit}(u)-L_\texttt{fit}(\hat{u})\}$ and at most $L_\texttt{fit}(u)-\underline{L}$ suboptimal.
This allows us to evaluate candidates returned by an uncertified optimizer, such as Chinchilla Approach~3, without knowing the exact globally optimal loss. In \Cref{fig:suboptimality-vs-runtime}, we use the tightest certified interval obtained by our branch-and-bound search as a common reference, and compare candidates returned by both methods on $4$ real scaling ladders (see \Cref{app:datasets} for dataset details).

\paragraph{Runtime.} We can loosely bound the runtime and memory complexity of \Cref{alg:certified-fit}.
The analysis in \Cref{remark:certified-fitting-complexity} gives an $\mathcal O(\epsilon^{-d})$ worst-case bound on the number of branching iterations as $\epsilon\to0$, for fixed $d=2k+1$.
Under mild conditions on how the problem scales with $k$ (\Cref{remark:certified-fitting-complexity-in-dimension}), the bound is at most $(Cd/\epsilon)^d$ for a constant $C$ independent of $d$ and $\epsilon$.
This deliberately conservative bound corresponds to exploring a complete subdivision cover of the space, and does not account for regions eliminated by pruning. 

Empirically, \Cref{fig:suboptimality-vs-runtime} shows the suboptimality of the solutions returned from the branch-and-bound search vs. runtime, as compared to the Chinchilla Approach~3 (while varying the initialization grid for Chinchilla). Increasing the initialization-grid density of Chinchilla Approach~3 can also produce better candidates, but does \textit{not} give any guarantees on the returned solution or guarantee that it will improve upon one from a coarser grid.
The wall-clock comparison should be interpreted as illustrative rather than as an implementation-normalized benchmark: our branch-and-bound implementation is sequential and implemented in \texttt{C++}; on the other hand, the Chinchilla baseline runs SciPy L-BFGS sequentially from every initialization.
The Chinchilla algorithm can be readily parallelized without changing the fitting procedure; parallel processing of branch-and-bound boxes is also possible, but is not part of \Cref{alg:certified-fit} as stated. Hence, in practice, Chinchilla can be run much more efficiently than \Cref{fig:suboptimality-vs-runtime} might suggest. Nonetheless, \Cref{fig:suboptimality-vs-runtime} speaks to the tractability of running \Cref{alg:certified-fit} efficiently in practice.

\paragraph{Memory.}
We can also loosely upper-bound the memory complexity of the algorithm, knowing its iteration complexity.
After $T$ branching iterations, the queue contains at most $T+1$ boxes (each iteration increases the queue size by at most one). Hence,  the iteration bound of \Cref{remark:certified-fitting-complexity}, $\mathcal O(\epsilon^{-d})$ for fixed $d$, also bounds the memory complexity of the algorithm. 
This is even more conservative a bound than for runtime, since many boxes might be discarded, and so the actual peak number of stored boxes might be much smaller than the total number of iterations.

%
\begin{figure}
    \centering
    \includegraphics[width=1.0\textwidth]{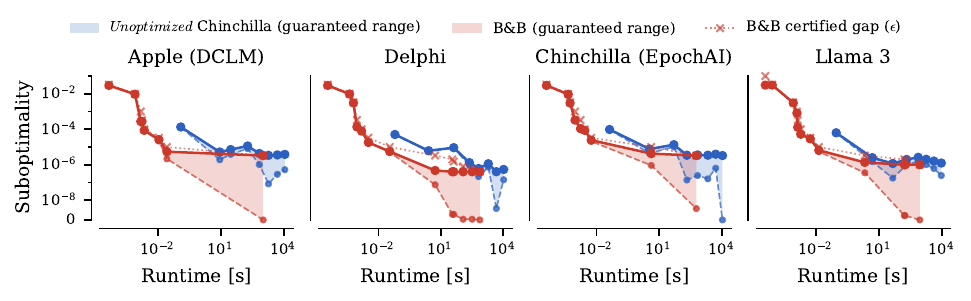}
    \caption{\textbf{Certified suboptimality of Chinchilla Approach~3 vs. Branch and Bound \Cref{alg:certified-fit} against wall-clock runtime.}
    Comparison of certified branch-and-bound fitting (B\&B; red) and Chinchilla Approach~3 (blue) on four 5-parameter scaling-law datasets.
    For a solution returned by a fitting method with fitting loss $L$, shaded regions show the guaranteed interval within which the ``suboptimality'' $L - \min_\theta L_\texttt{fit}(\theta; \calD)$ must lie (as estimated using the tightest bounds obtained from the branch-and-bound search).
    Chinchilla points give the best candidate obtained using different $K^5$ initialization-grid densities for $K=1,\dots,8$. Runtimes are wall-clock measurements from sequential implementations: C++ for B\&B and SciPy L-BFGS for Chinchilla Approach~3.  The runtime comparison is only illustrative as the two implementations have not been optimized to an equal extent.
    }
    \label{fig:suboptimality-vs-runtime}
\end{figure}



\subsection{Certified Sublevel Sets Algorithm}
\label{sec:viable-set-algorithm}

A certified optimizer returns one nearly optimal fit but does not reveal whether the data distinguish that fit from other parameter settings.
Instead of asking for just one fit, we can also ask explicitly for some representation of the entire ``viable'' set $\mathcal S_\epsilon$ of parameters in $\Theta$ that are optimal to within an $\epsilon$-threshold. In the transformed coordinates, this set is
\begin{equation}
    \mathcal S^\phi_\epsilon
    \coloneqq\phi(\mathcal S_\epsilon)
    =\{u\in\mathcal U:L_{\texttt{fit}}(u;\mathcal D)
    \leq\min_{u'\in\mathcal U}L_{\texttt{fit}}(u';\mathcal D)+\epsilon\}.
    \label{eq:transformed-sublevel-set}
\end{equation}
Returning this exact set is intractable; instead, in \Cref{alg:viable-set-cover}, we give an algorithm that constructs a collection of boxes $\mathcal B$ satisfying $\mathcal S^\phi_{\epsilon/2}\subseteq\bigcup_{B\in\mathcal B}B\subseteq\mathcal S^\phi_\epsilon$. Thus, the cover contains every fit within $\epsilon/2$ of optimal, while every parameter it contains is guaranteed to be within $\epsilon$ of optimal.\footnote{The algorithm can be generalized to an arbitrary inner fitting-loss tolerance $\rho\geq0$ and slack $\eta>0$, where the returned box cover satisfies $\mathcal S^\phi_{\rho}\subseteq\bigcup_{B\in\mathcal B}B\subseteq\mathcal S^\phi_{\rho+\eta}$. The stated guarantee corresponds to $\rho=\eta=\epsilon/2$.}
Mapping the union back to the original parameters gives
\begin{equation}
    \mathcal C\coloneqq\phi^{-1}\!\left(\bigcup_{B\in\mathcal B}B\right),
    \qquad
    \mathcal S_{\epsilon/2}\subseteq\mathcal C\subseteq\mathcal S_\epsilon,
    \label{eq:original-sublevel-cover}
\end{equation}
as stated in \Cref{thm:main}(b). The mapped regions $\phi^{-1}(B)$ will generally not be boxes.

\begin{algorithm}[!t]
\SetAlgoLined
\KwIn{Search domain $\mathcal U$; dataset $\mathcal D$; viable tolerance $\epsilon>0$}
\KwOut{Box cover $\mathcal B$ such that
$\mathcal S^\phi_{\epsilon/2}\subseteq\bigcup_{B\in\mathcal B}B\subseteq\mathcal S^\phi_\epsilon$,
where
$\mathcal S^\phi_\epsilon
=
\{u\in\mathcal U:
L_{\texttt{fit}}(u;\mathcal D)
\leq
\inf_{u'\in\mathcal U}L_{\texttt{fit}}(u';\mathcal D)+\epsilon\}$}

\tcp{Phase 1: Certify the global minimum to within $\frac{\epsilon}{3}$}
$(\hat u,\underline L)
\gets
\textsc{CertifiedFittingSearch}
(\mathcal U,\epsilon/3,\mathcal D)$ \tcp*[f]{Get estimate $\underline{L}$ within $\frac{\epsilon}{3}$ of optimal}\;

\tcp{Phase 2: Construct the viable-set cover}
Initialize a queue $\mathcal Q\gets\{\mathcal U\}$\;
Initialize the viable cover $\mathcal B\gets\emptyset$\;

\While{$\mathcal Q$ is not empty}{
    $B\gets\mathcal Q.\textsc{Pop}()$\;
    Compute $\underline L_{\texttt{fit}}(B)$
    and $\overline L_{\texttt{fit}}(B)$\;

    \If(\tcp*[f]{$B$ contains no point of $\mathcal S^\phi_{\epsilon/2}$}){
        $\underline L_{\texttt{fit}}(B)
        >
        L_{\texttt{fit}}(\hat u;\mathcal D)+\epsilon/2$
    }{
        \textbf{continue}\;
    }

    \If(\tcp*[f]{$B$ is fully contained in $\mathcal S^\phi_\epsilon$}){
        $\overline L_{\texttt{fit}}(B)
        \leq
        \underline L+\epsilon$
    }{
        $\mathcal B\gets\mathcal B\cup\{B\}$\;
        \textbf{continue}\;
    }

    \ForEach{child box $C\in\textsc{BranchRegion}(B)$}{
        $\mathcal Q.\textsc{Push}(C)$\;
    }
}

\KwRet $\mathcal B$\;
\caption{Certified sublevel set cover}
\label{alg:viable-set-cover}
\end{algorithm}

\Cref{fig:sivia-viable-set} illustrates the resulting cover.  The left panels show that the near-optimal set is an elongated, curved region in parameter space.  The right panel maps this region into the associated family of scaling law curves which diverges as the scale $x$ grows.  We use a large tolerance in this figure so that the shape of the set is visible; Figure~\ref{fig:non-identifiability-scaling-law} shows that such divergence persists at tolerances within 1\% of the optimal fitting loss. This motivates the prediction-envelope search in the next section.

\begin{figure}[!t]
    \centering
    \includegraphics[width=\textwidth]{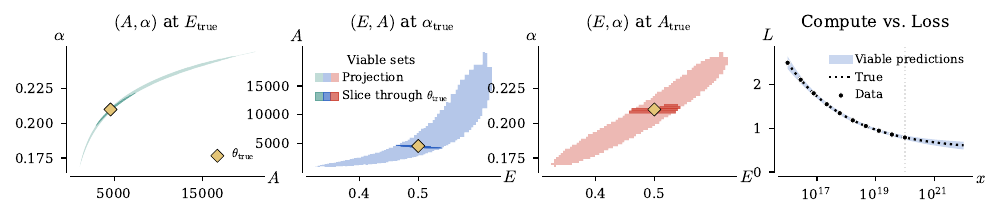}
    \caption{\textbf{Sublevel set cover and prediction envelope from certified search.}
    Illustration of the output of the sublevel set cover algorithm (\Cref{alg:viable-set-cover}) on a synthetic single-variable scaling law $L(x;\theta) = E + A x^{-\alpha}$, fitted with the log-space Huber loss of \Cref{sec:chinchilla-fitting} ($\delta = 10^{-3}$), at inner and outer fitting-loss tolerances $\epsilon/2 = 10^{-5}$ and $\epsilon = 2\times10^{-5}$.  The certified minimum fitting loss is approximately $2.3\times 10^{-7}$, so $\epsilon$ is roughly 90 times the optimum.  The algorithm returns a collection of boxes $\mathcal B$ in $(E,C,r)$ coordinates. Its mapped union $\mathcal C=\phi^{-1}(\bigcup_{B\in\mathcal B}B)$ satisfies $\mathcal S_{\epsilon/2}\subseteq\mathcal C\subseteq\mathcal S_\epsilon$ in the original parameters.
    \textbf{(Left three panels)} Two-dimensional projections of $\mathcal C$ onto pairs of original coordinates $(E,A,\alpha)$, shown as light-shaded regions, and cross-sections with the omitted coordinate fixed at its ground-truth value, shown as dark-shaded regions. The diamonds mark the ground-truth parameters $\theta_{\mathrm{true}}$.
    \textbf{(Right panel)} The corresponding prediction envelope: the shaded band shows the range of scaling curves $L(x;\theta)$ achievable by any $\theta\in\mathcal C$.
    }
    \label{fig:sivia-viable-set}
\end{figure}

\paragraph{Runtime \& memory.} The worst-case bounds for \Cref{alg:certified-fit} also apply to \Cref{alg:viable-set-cover}, up to constants: a box is subdivided only while its loss gap exceeds $\epsilon/6$, so the subdivision tree has depth $\mathcal O(d\log(1/\epsilon))$, and Phase~2 performs $\mathcal O(\epsilon^{-d})$ iterations for fixed $d$ (\Cref{app:proofs}).
In practice, however, \Cref{alg:viable-set-cover} is far more expensive. \Cref{alg:certified-fit} stops as soon as its incumbent is within $\epsilon$ of the smallest remaining lower bound, and pruning typically means that it explores only a small part of the tree. \Cref{alg:viable-set-cover} must instead classify every box as inside $\mathcal S^\phi_\epsilon$ or outside $\mathcal S^\phi_{\epsilon/2}$, so it refines every box that straddles the boundary of the sublevel set until its loss gap is at most $\epsilon/6$.
That the two algorithms share a worst-case bound but differ greatly in observed runtime indicates how conservative the bound is.
In practice, we find that we can run this algorithm for \textit{reasonable} choices of $\epsilon$, but not for thresholds comparable to the suboptimality of solutions from Chinchilla Approach~3.

\subsection{Certified Prediction Envelopes}\label{sec:certified-prediction-envelopes}

The certified sublevel set cover characterizes all nearly optimal parameter settings, but it can be computationally expensive. Ultimately, when using scaling laws in practice for extrapolation, we primarily care about the resulting uncertainty in scaling law \textit{predictions} that ambiguous fitting causes. It turns out that, rather than trying to characterize the entire viable set $\mathcal{S}^\phi_\epsilon$, we can instead much more cheaply characterize just the range of predictions at a given input $x^\ast$ that the viable set would produce.


Concretely, fix a target fitting-loss tolerance $\epsilon>0$, a loss-tolerance slack $\eta>0$, and a target input $x^\ast$. For any fitting-loss tolerance $\rho\geq0$, define
\[
    \mathcal L_{x^\ast,\rho}
    \coloneqq
    \{L(x^\ast;u):u\in\mathcal S^\phi_\rho\}
    =\{L(x^\ast;\theta):\theta\in\mathcal S_\rho\}.
\]
Below, we use this definition at two tolerance levels: $\rho=\epsilon$ and $\rho=\epsilon+\eta$.
Our goal is to search for an interval that contains every prediction arising from $\mathcal S^\phi_\epsilon$;
we do not want the interval to be too loose, so we also require that the endpoints be attained by parameters in the slightly enlarged set $\mathcal S^\phi_{\epsilon+\eta}$.
For instance, for the upper endpoint, this means returning $\overline L_{x^\ast}$ satisfying \( \sup\mathcal L_{x^\ast,\epsilon}
    \leq
    \overline L_{x^\ast}
    \leq
    \sup\mathcal L_{x^\ast,\epsilon+\eta}
\).

\Cref{alg:certified-prediction-bound} computes such an upper endpoint directly.
It first runs \Cref{alg:certified-fit} with optimality tolerance $\eta/2$ to obtain certified lower and upper bounds on the globally optimal fitting loss.
We then run a branch and bound loop that greedily prioritizes boxes with the largest possible prediction at the target input $x^\ast$.
By using the monotonic reparameterization, we know that the largest prediction at $x^\ast$ for this parameter box will be attained in one of the corners.
We can discard boxes if they are completely disjoint from $\mathcal{S}^\phi_\epsilon$ (the fitting-loss lower bound determines whether $B$ could contain a fit in $\mathcal S^\phi_\epsilon$), or subdivide them further if that is still unclear.
It terminates when the highest-priority remaining box is certified to lie entirely within $\mathcal S^\phi_{\epsilon+\eta}$.
Its prediction-maximizing corner then upper-bounds every prediction arising from $\mathcal S^\phi_\epsilon$, while itself belonging to $\mathcal S^\phi_{\epsilon+\eta}$.
The corresponding algorithm for a \textit{lower bound} on the prediction can be obtained with minor adjustments to \Cref{alg:certified-prediction-bound}.\footnote{Phase~1 can be shared between the lower and upper envelope searches.}


For clarity, below we assume that we are only interested in predictions at points $x^\ast$ for which $x_i^\ast\geq x_i^{\min}$ for every coordinate, so that the upper corner of $B$ maximizes prediction over $u \in B$.
This is not strictly necessary --- the general algorithm, which selects the appropriate corner separately for each coordinate, is given in \Cref{app:arbitrary-target-prediction-bounds,alg:certified-prediction-bound-general} --- but it simplifies the exposition.

\begin{algorithm}[t]
\SetAlgoLined
\KwIn{
Search domain $\mathcal U=[\mathcal U^{\lowercorner}, \mathcal U^{\uppercorner}]$; dataset $\mathcal D$; target scale $x^\ast$;
inner loss tolerance $\epsilon>0$; loss tolerance slack $\eta>0$;
}
\KwOut{Prediction bound $\overline L_{x^\ast}$ satisfying
\[
    \sup\mathcal L_{x^\ast,\epsilon}
    \leq
    \overline L_{x^\ast}
    \leq
    \sup\mathcal L_{x^\ast,\epsilon + \eta} \qquad
    \text{where } \mathcal L_{x^\ast,\epsilon} \coloneqq \{L(x^\ast; u): u \in \mathcal{S}^\phi_{\epsilon}\}.
\]
}

\tcp{Phase 1: certify the globally optimal fitting loss}
$(\hat u,\underline L) \gets \textsc{CertifiedFittingSearch}
(\mathcal U,\eta/2,\mathcal D)$\;
$L_{\mathrm{fit}}^{\mathrm{ub}}
\gets
L_{\texttt{fit}}(\hat u;\mathcal D)$
\tcp*[f]{We have 
$\underline L
\leq L_{\texttt{fit}}^\star
\leq L_{\mathrm{fit}}^{\mathrm{ub}}
\leq \underline L+\eta/2$
}

\tcp{Phase 2: search for the largest viable prediction}
Initialize a max-priority queue $\mathcal Q$\;
$\mathcal Q.\textsc{Push}
\bigl(
    \mathcal U,
    \text{priority}=L(x^\ast;\mathcal U^{\uppercorner})
\bigr)$
\tcp*[f]{Queue ordered in terms of max prediction at $x^\ast$}
\;

\While{$\mathcal Q$ is not empty}{
    $B\gets\mathcal Q.\textsc{PopMax}()$
    \tcp*[f]{Box with the largest remaining prediction upper bound}\;

    Compute $\underline L_{\texttt{fit}}(B)$
    and $\overline L_{\texttt{fit}}(B)$\;


    \If(\tcp*[f]{$B$ is entirely contained in
        $\mathcal S^\phi_{\epsilon+\eta}$}){
        $\overline L_{\texttt{fit}}(B)
        \leq
        \underline L+\epsilon+\eta$
    }{
        \KwRet $L(x^\ast;B^{\uppercorner})$\;
    }
    \ForEach{child box $C\in\textsc{BranchRegion}(B)$}{
        Compute $\underline L_{\texttt{fit}}(C)$\;
    
        \If{
            $\underline L_{\texttt{fit}}(C)
            \leq
            L_{\mathrm{fit}}^{\mathrm{ub}}+\epsilon$
        }{
            $\mathcal Q.\textsc{Push}
            \bigl(
                C,
                \text{priority}=L(x^\ast;C^{\uppercorner})
            \bigr)$\;
        }
    }
}
\caption{Certified upper prediction bound with viability slack}
\label{alg:certified-prediction-bound}
\end{algorithm}

Running \Cref{alg:certified-prediction-bound} for the upper and lower bound returns an interval containing the range of predictions at a given scale $x^\ast$, and repeating this process across many target scales can produce a prediction envelope, such as that seen in \Cref{fig:non-identifiability-scaling-law}.  The resulting band is a pointwise, certified range over viable fits: every viable scaling curve lies inside it, while each endpoint is attained by a fit in $\mathcal S^\phi_{\epsilon+\eta}$.

\paragraph{Runtime \& Memory.}  The same iteration and queue-memory bounds can again be given as for \Cref{alg:certified-fit,alg:viable-set-cover}, using the arguments in \Cref{remark:certified-fitting-complexity,remark:certified-fitting-complexity-in-dimension} with accuracy parameter $\eta/2$ for both phases. Namely, for fixed $d$, the algorithm performs $\mathcal O(\eta^{-d})$ branching iterations and stores $\mathcal O(\eta^{-d})$ boxes as $\eta\to0$.
The bound does not depend on $\epsilon$. A box is subdivided only if it passes the pruning test $\underline L_{\texttt{fit}}(B)\le L^{\mathrm{ub}}_{\mathrm{fit}}+\epsilon$ but fails the acceptance test, i.e.\ $\overline L_{\texttt{fit}}(B)>\underline L+\epsilon+\eta$. Subtracting, $\epsilon$ cancels, and the loss gap of the box exceeds $\eta-(L^{\mathrm{ub}}_{\mathrm{fit}}-\underline L)\ge\eta/2$ by the Phase~1 certificate. Thus $\eta$ alone determines the depth of the subdivision tree.
Because boxes are prioritized by their largest possible prediction, we find empirically that \Cref{alg:certified-prediction-bound} runs much faster than the sublevel set search in \Cref{sec:viable-set-algorithm} (\Cref{alg:viable-set-cover}).
In practice, \Cref{alg:certified-prediction-bound} is fast for three-parameter scaling laws: for the envelope in \Cref{fig:non-identifiability-scaling-law}, at a tolerance of $1\%$ of the optimal loss, each endpoint took under a minute to certify (\Cref{sec:experimental-details}).
The five-parameter laws in \Cref{fig:viable-extrap-fans-all} are more demanding: there, the Chinchilla fits are within $10^{-5}$ of the optimal loss, and certifying envelopes at such small tolerances is beyond our current compute budget. We therefore construct the bands in \Cref{fig:viable-extrap-fans-all} by local search (\Cref{sec:experimental-details}).

\subsection{Limitations and Future Work}
\label{sec:limitations}
\paragraph{Computational cost.}
Certifying global optimality is computationally hard for nonconvex problems in general, and the worst-case cost of our algorithms grows exponentially in the number of parameters. Certified fitting is practical for the scaling laws we consider, but certified prediction envelopes for five-parameter laws at tight tolerances currently exceed our compute budget, which is why \Cref{fig:viable-extrap-fans-all} relies on local search.
Scaling law fitting has more structure than a generic nonconvex problem: the number of parameters is small, and predictions are monotone after reparameterization. Understanding the complexity of this specific problem, and in particular how it depends on the scaling ladder, is an interesting open question. On the practical side, adaptive splitting rules, tighter loss bounds, and parallel implementations could substantially reduce runtime and memory.

\paragraph{Randomness in the observations.}
In the scaling ladders we consider, each observation is the validation loss of a single training run, with one initialization seed and one data seed, averaged over a fixed evaluation set. Each observation is therefore the realization of a random variable, but our analysis, like standard fitting practice, applies the fitting procedure to a single realization of the ladder without accounting for this randomness. The randomness can change the outcome of the fitting procedure, including the fitted law and its sublevel set. The choice of statistic matters as well: the loss of a single run, its average over seeds, and other summaries of the per-example losses need not follow the same scaling law. Understanding how these sources of variability enter scaling law estimation is an important direction for future work.

\paragraph{Choice of objective.}
We follow standard practice in minimizing a log-residual Huber loss. Our bounds can be extended to most other residual losses, with some minor adjustments to the above algorithms, but which objective is most appropriate for scaling law fitting, and how it should reflect the noise in the observations, remains open.

\paragraph{Loss landscape.}
The sublevel sets we find are broad, and the predictions they imply diverge substantially. Understanding the geometry of these loss landscapes, and in particular which features of a scaling ladder make the sublevel set broad or narrow, would clarify when extrapolations can be trusted and could inform how scaling ladders are designed.

\section{Related Work}
\label{sec:related-work}

\paragraph{Scaling laws and fitting.}
Power laws have long been used to describe how neural-network performance varies with model size, dataset size, and training compute. In language modeling, \citet{kaplan2020scaling} proposed a coupled model--data scaling surface $L((N, D);\theta) = \left(A_N /N^{\alpha_N / \alpha_D} + A_D / D \right)^{\alpha_D}$, while \citet{hoffmann2022training} introduced the now widely used additive form
\(
L(N,D)=E+A N^{-\alpha}+B D^{-\beta}.
\)
Subsequent work has applied or extended these families to new model suites, downstream metrics, data mixtures, distillation, and multimodal or diffusion settings, sometimes adding further covariates or a coupling exponent~\citep{gadre2025gemstones,apple2025revisiting,shukor2025scaling,busbridge2025distillationscalinglaws,bethune2026designspacetrimodalmasked,videau2026skalingchinchillasexponentsmeet,ramapuram2026scalingpropertiescontinuousdiffusion}.
Fitting practices are correspondingly heterogeneous~\citep{choshen2024hitchhiker,li2025misfitting}. They include linear regression in log--log coordinates, quadratic fits to IsoFLOP profiles followed by a power-law regression of their estimated optima (Chinchilla Approach~2), and direct nonlinear fitting of the complete loss surface (Approach~3)~\citep{hoffmann2022training,dubey2024llama}. Approach~3 reparameterizes positive coefficients in log coordinates and minimizes a Huber loss on log-residuals using L-BFGS from a grid of initializations, retaining the local solution with the smallest observed objective.

\paragraph{Problems with scaling-law estimation.}
Recent work has shown that seemingly minor choices throughout this pipeline can materially change the resulting law.
\citet{besiroglu2024chinchilla} reanalyzed the Chinchilla data and found that the originally reported Approach~3 parameters fit reconstructed observations poorly and that premature optimizer termination during bootstrapping produced implausibly narrow confidence intervals.
Other discrepancies have been traced to FLOP accounting, warmup duration, scale-dependent hyperparameter tuning~\citep{porian2024resolving}, and ambiguity in which model parameters are counted. \citet{schaeffer2025evaluating} point out that some conclusions in the Chinchilla paper, such as the compute-optimal data-to-parameter ratio, remain fairly robust to these discrepancies. \citet{li2025misfitting} document substantial variation and frequent underreporting in functional forms, fitting objectives, optimizers, and initialization schemes.
Finally, \citet{czech2026problems} show that the parabolic approximation underlying Chinchilla Approach~2 (which is used for estimating the compute-optimal data-to-parameter ratio) can introduce systematic bias under wide or off-center IsoFLOP sampling and asymmetric scaling exponents; this is perhaps unsurprising, given that Chinchilla Approach~2 attempts to fit a quadratic to data from a curve that even in the idealized setting is far from a quadratic.  These works attribute variation in fitted laws to choices in the estimation pipeline or to optimization error.  We study a different question: whether fits with nearly the same fitting loss can make substantially different predictions.

\paragraph{Improved fitting and experimental design.}
Several recent approaches improve particular stages of scaling-law estimation.
\citet{shukor2025scaling} replace a fixed grid of L-BFGS initializations with basin hopping, using randomized moves between local minima and L-BFGS as an inner solver. This can locate better fits more efficiently in high-dimensional scaling laws, but remains a heuristic; at the end of the process, the practitioner is still left unsure as to how well the fitting loss was actually optimized.
For the squared-error objective in untransformed loss space, \citet{czech2026problems} exploit partial linearity through Variable Projection with Non-negative Least Squares (VPNLS): for the scaling law in \Cref{eq:chinchilla-power-law-form}, fixing the exponents $(\alpha_1,\dots,\alpha_k)$ allows the coefficients $(E,A_1,\dots, A_k)$ to be determined by least-squares regression, reducing the nonlinear search to the exponent dimensions.
This reduction does not directly apply to the standard log-residual Huber objective, however.
Anecdotally, we observe that using the log-residual loss yields lower extrapolation error than using MSE or Huber loss in untransformed residual space.
Hence, the faster and more robust optimization from VPNLS comes at the cost of having to use a potentially worse fitting objective.
Complementary work improves the data supplied to the fit: \citet{choshen2024hitchhiker} study which models and checkpoints make extrapolation more reliable. For sequential experimental design, \citet{li2026spend} use an approximate posterior over multiple fitting basins to select runs by their expected reduction in target-region prediction uncertainty relative to cost.

\paragraph{The Rashomon effect.}  Following~\citet{breiman2001}, the set of models whose empirical loss is within a tolerance of the best, often called the Rashomon set, has been used to study variable importance~\citep{fisher2019all} and the variation of predictions across near-optimal models~\citep{marx2020predictive,watsondaniels2023predictive}.  These quantities can be computed exactly by integer programming for linear classifiers~\citep{marx2020predictive}, by optimization for convex losses~\citep{watsondaniels2023predictive}, and by enumeration for sparse decision trees~\citep{xin2022exploring}.  The sublevel set $\mathcal{S}_\epsilon$ is the Rashomon set of the scaling law function class.  Our algorithms give certified approximations to this set and to the range of predictions it induces, for a continuous, nonconvex model class.

\paragraph{Global non-convex optimization and interval analysis.}
Our algorithms draw on deterministic global optimization using branch-and-bound methods and set inversion.
Broadly speaking, many branch-and-bound methods recursively partition a search domain, compute lower bounds over each region, maintain feasible incumbents as upper bounds, and discard regions that cannot improve the incumbent~\citep{lawler1966branch,horst1996global,floudas2000deterministic}.
Interval global optimization applies the same principle using rigorous enclosures of function values over parameter boxes~\citep{hansen2003global}.
Relatedly, \citet{jaulin1993set} introduced Set Inversion Via Interval Analysis (SIVIA) to construct guaranteed inner and outer approximations of parameter sets consistent with bounded-error observations. This approach inspired our viable-set cover algorithm (\Cref{alg:viable-set-cover}).
\section{Conclusion}
\label{sec:conclusion}
The practice of fitting scaling laws occupies a central role in AI research, guiding resource allocation and architectural decisions for frontier models. Yet the optimization process that translates empirical scaling ladders into parametric curves has largely been treated as a routine heuristic. In this work, we have demonstrated that this fitting stage introduces ambiguities that can substantially undermine the reliability of scaling forecasts.

Standard fitting procedures run local optimizers from a grid of initializations and return whichever parameters achieved the lowest loss on that grid. Because the loss landscape is non-convex, the returned fit carries no guarantee of being nearly optimal; and because the sublevel sets of near-minimizers are broad, even a nearly optimal fit is largely an arbitrary selection from the $\eps$-sublevel set.  By contrast, our framework certifies how close a fit is to optimal and does not attempt to select a single, favored curve. Instead, we provide procedures for tightly bounding the $\eps$-sublevel set and for determining the complete envelope of predictions made by near-optimal fits. This shifts the focus from finding an illusory unique fit to quantifying the uncertainty introduced by the fitting stage.

Scaling law fitting is only the final stage of scaling law estimation. While the tools developed in this work allow practitioners to compute the certified prediction envelope for a given dataset, the width of that envelope is ultimately governed by the earlier stages of the pipeline---the design of the scaling ladder, the training procedures, and the choice of evaluation metrics---that generate the data.  Improving the reliability of scaling law forecasts requires optimizing each stage of this pipeline.

\bibliography{references}

\begin{thebibliography}{33}
\providecommand{\natexlab}[1]{#1}
\providecommand{\url}[1]{\texttt{#1}}
\expandafter\ifx\csname urlstyle\endcsname\relax
  \providecommand{\doi}[1]{doi: #1}\else
  \providecommand{\doi}{doi: \begingroup \urlstyle{rm}\Url}\fi

\bibitem[Besiroglu et~al.(2024)Besiroglu, Erdil, Barnett, and You]{besiroglu2024chinchilla}
T.~Besiroglu, E.~Erdil, M.~Barnett, and J.~You.
\newblock Chinchilla scaling: A replication attempt.
\newblock \emph{arXiv preprint arXiv:2404.10102}, 2024.
\newblock URL \url{https://arxiv.org/abs/2404.10102}.

\bibitem[Bethune et~al.(2026)Bethune, Turrisi, Mlodozeniec, Lopez, Boominathan, Bhendawade, Shidani, Pelemans, Olausson, Hjelm, Dixon, Monteiro, Ablin, Banna, Blaas, Henderson, Noriy, Busbridge, Susskind, Cuturi, Belousova, Zappella, Webb, and Ramapuram]{bethune2026designspacetrimodalmasked}
L.~Bethune, V.~Turrisi, B.~K. Mlodozeniec, P.~R. Lopez, L.~Boominathan, N.~Bhendawade, A.~Shidani, J.~Pelemans, T.~X. Olausson, D.~Hjelm, P.~Dixon, J.~Monteiro, P.~Ablin, V.~Banna, A.~Blaas, N.~Henderson, K.~Noriy, D.~Busbridge, J.~Susskind, M.~Cuturi, I.~Belousova, L.~Zappella, R.~Webb, and J.~Ramapuram.
\newblock The design space of tri-modal masked diffusion models, 2026.
\newblock URL \url{https://arxiv.org/abs/2602.21472}.

\bibitem[Breiman(2001)]{breiman2001}
L.~Breiman.
\newblock {Statistical Modeling: The Two Cultures (with comments and a rejoinder by the author)}.
\newblock \emph{Statistical Science}, 16\penalty0 (3):\penalty0 199 -- 231, 2001.

\bibitem[Busbridge et~al.(2025)Busbridge, Shidani, Weers, Ramapuram, Littwin, and Webb]{busbridge2025distillationscalinglaws}
D.~Busbridge, A.~Shidani, F.~Weers, J.~Ramapuram, E.~Littwin, and R.~Webb.
\newblock Distillation scaling laws, 2025.
\newblock URL \url{https://arxiv.org/abs/2502.08606}.

\bibitem[Choshen et~al.(2024)Choshen, Zhang, and Andreas]{choshen2024hitchhiker}
L.~Choshen, Y.~Zhang, and J.~Andreas.
\newblock A hitchhiker's guide to scaling law estimation.
\newblock \emph{arXiv preprint arXiv:2410.11840}, 2024.
\newblock URL \url{https://arxiv.org/abs/2410.11840}.

\bibitem[Czech et~al.(2026)Czech, Xu, Elmatad, Wang, and Held]{czech2026problems}
E.~Czech, Z.~Xu, Y.~Elmatad, Y.~Wang, and W.~Held.
\newblock Problems with {Chinchilla} approach 2: Systematic biases in {IsoFLOP} parabola fits.
\newblock \emph{arXiv preprint arXiv:2603.22339}, 2026.
\newblock URL \url{https://arxiv.org/abs/2603.22339}.

\bibitem[Dubey et~al.(2024)Dubey, Grattafiori, Jauhri, et~al.]{dubey2024llama}
A.~Dubey, A.~Grattafiori, A.~Jauhri, et~al.
\newblock The llama 3 herd of models, 2024.
\newblock URL \url{https://arxiv.org/abs/2407.21783}.

\bibitem[Fisher et~al.(2019)Fisher, Rudin, and Dominici]{fisher2019all}
A.~Fisher, C.~Rudin, and F.~Dominici.
\newblock All models are wrong, but many are useful: Learning a variable's importance by studying an entire class of prediction models simultaneously.
\newblock \emph{Journal of Machine Learning Research}, 20\penalty0 (177):\penalty0 1--81, 2019.

\bibitem[Floudas(2000)]{floudas2000deterministic}
C.~A. Floudas.
\newblock \emph{Deterministic Global Optimization: Theory, Methods and Applications}.
\newblock Kluwer Academic Publishers, Dordrecht, 2000.
\newblock \doi{10.1007/978-1-4757-5025-6}.

\bibitem[Hansen and Walster(2003)]{hansen2003global}
E.~Hansen and G.~W. Walster.
\newblock \emph{Global Optimization Using Interval Analysis}.
\newblock Marcel Dekker, New York, 2nd edition, 2003.
\newblock \doi{10.1201/9780203026939}.

\bibitem[Held et~al.(2026)Held, {Open Athena}, and {Marin Community}]{marin2026delphi}
W.~Held, {Open Athena}, and {Marin Community}.
\newblock Scaling laws that extrapolate 300x past the fit ({Delphi}).
\newblock \url{https://openathena.ai/blog/delphi/}, 2026.
\newblock Technical Report / OpenAthena Blog.

\bibitem[Hoffmann et~al.(2022)Hoffmann, Borgeaud, Mensch, Buchatskaya, Cai, Rutherford, de~Las~Casas, Hendricks, Welbl, Clark, Hennigan, Noland, Millican, {van den Driessche}, Damoc, Guy, Osindero, Simonyan, Elsen, Rae, Vinyals, and Sifre]{hoffmann2022training}
J.~Hoffmann, S.~Borgeaud, A.~Mensch, E.~Buchatskaya, T.~Cai, E.~Rutherford, D.~de~Las~Casas, L.~A. Hendricks, J.~Welbl, A.~Clark, T.~Hennigan, E.~Noland, K.~Millican, G.~{van den Driessche}, B.~Damoc, A.~Guy, S.~Osindero, K.~Simonyan, E.~Elsen, J.~W. Rae, O.~Vinyals, and L.~Sifre.
\newblock Training compute-optimal large language models.
\newblock \emph{arXiv preprint arXiv:2203.15556}, 2022.
\newblock URL \url{https://arxiv.org/abs/2203.15556}.

\bibitem[Horst and Tuy(1996)]{horst1996global}
R.~Horst and H.~Tuy.
\newblock \emph{Global Optimization: Deterministic Approaches}.
\newblock Springer-Verlag, Berlin, Heidelberg, 3rd edition, 1996.
\newblock \doi{10.1007/978-3-662-03199-5}.

\bibitem[Huber(1964)]{huber1964robust}
P.~J. Huber.
\newblock {Robust Estimation of a Location Parameter}.
\newblock \emph{The Annals of Mathematical Statistics}, 35\penalty0 (1):\penalty0 73 -- 101, 1964.
\newblock \doi{10.1214/aoms/1177703732}.
\newblock URL \url{https://doi.org/10.1214/aoms/1177703732}.

\bibitem[Jaulin and Walter(1993)]{jaulin1993set}
L.~Jaulin and E.~Walter.
\newblock Set inversion via interval analysis for nonlinear bounded-error estimation.
\newblock \emph{Automatica}, 29\penalty0 (4):\penalty0 1053--1064, 1993.
\newblock \doi{10.1016/0005-1098(93)90106-4}.

\bibitem[Kaplan et~al.(2020)Kaplan, McCandlish, Henighan, Brown, Chess, Child, Gray, Radford, Wu, and Amodei]{kaplan2020scaling}
J.~Kaplan, S.~McCandlish, T.~Henighan, T.~B. Brown, B.~Chess, R.~Child, S.~Gray, A.~Radford, J.~Wu, and D.~Amodei.
\newblock Scaling laws for neural language models.
\newblock \emph{arXiv preprint arXiv:2001.08361}, 2020.
\newblock URL \url{https://arxiv.org/abs/2001.08361}.

\bibitem[Krajewski et~al.(2025)Krajewski, Shidani, Busbridge, Wiseman, and Ramapuram]{apple2025revisiting}
J.~Krajewski, A.~Shidani, D.~Busbridge, S.~Wiseman, and J.~Ramapuram.
\newblock Revisiting the scaling properties of downstream metrics in large language model training.
\newblock \emph{arXiv preprint arXiv:2512.08894}, 2025.
\newblock URL \url{https://arxiv.org/abs/2512.08894}.

\bibitem[Lawler and Wood(1966)]{lawler1966branch}
E.~L. Lawler and D.~E. Wood.
\newblock Branch-and-bound methods: A survey.
\newblock \emph{Operations Research}, 14\penalty0 (4):\penalty0 699--719, 1966.
\newblock \doi{10.1287/opre.14.4.699}.

\bibitem[Li et~al.(2024)Li, Fang, Smyrnis, Ivgi, Jordan, Gadre, Bansal, Guha, Keh, Arora, Garg, Xin, Muennighoff, Heckel, Mercat, Chen, Gururangan, Wortsman, Albalak, Bitton, Nezhurina, Abbas, Hsieh, Ghosh, Gardner, Kilian, Zhang, Shao, Pratt, Sanyal, Ilharco, Daras, Marathe, Gokaslan, Zhang, Chandu, Nguyen, Vasiljevic, Kakade, Song, Sanghavi, Faghri, Oh, Zettlemoyer, Lo, El-Nouby, Pouransari, Toshev, Wang, Groeneveld, Soldaini, Koh, Jitsev, Kollar, Dimakis, Carmon, Dave, Schmidt, and Shankar]{li2024datacomp}
J.~Li, A.~Fang, G.~Smyrnis, M.~Ivgi, M.~Jordan, S.~Gadre, H.~Bansal, E.~Guha, S.~Keh, K.~Arora, S.~Garg, R.~Xin, N.~Muennighoff, R.~Heckel, J.~Mercat, M.~Chen, S.~Gururangan, M.~Wortsman, A.~Albalak, Y.~Bitton, M.~Nezhurina, A.~Abbas, C.-Y. Hsieh, D.~Ghosh, J.~Gardner, M.~Kilian, H.~Zhang, R.~Shao, S.~Pratt, S.~Sanyal, G.~Ilharco, G.~Daras, K.~Marathe, A.~Gokaslan, J.~Zhang, K.~Chandu, T.~Nguyen, I.~Vasiljevic, S.~Kakade, S.~Song, S.~Sanghavi, F.~Faghri, S.~Oh, L.~Zettlemoyer, K.~Lo, A.~El-Nouby, H.~Pouransari, A.~Toshev, S.~Wang, D.~Groeneveld, L.~Soldaini, P.~W. Koh, J.~Jitsev, T.~Kollar, A.~G. Dimakis, Y.~Carmon, A.~Dave, L.~Schmidt, and V.~Shankar.
\newblock {DataComp-LM}: In search of the next generation of training sets for language models.
\newblock \emph{arXiv preprint arXiv:2406.11794}, 2024.
\newblock URL \url{https://arxiv.org/abs/2406.11794}.

\bibitem[Li et~al.(2025)Li, Kudugunta, and Zettlemoyer]{li2025misfitting}
M.~Li, S.~Kudugunta, and L.~Zettlemoyer.
\newblock {(Mis)Fitting}: A survey of scaling laws.
\newblock \emph{arXiv preprint arXiv:2502.18969}, 2025.
\newblock URL \url{https://arxiv.org/abs/2502.18969}.

\bibitem[Li et~al.(2026)Li, Li, Lin, Sun, Talwalkar, and Yang]{li2026spend}
S.~Li, S.~Li, H.~Lin, W.~Sun, A.~Talwalkar, and Y.~Yang.
\newblock Spend less, fit better: Budget-efficient scaling law fitting via active experiment selection.
\newblock \emph{arXiv preprint arXiv:2604.22753}, 2026.
\newblock URL \url{https://arxiv.org/abs/2604.22753}.

\bibitem[Madaan et~al.(2024)Madaan, Singh, Schaeffer, Poulton, Koyejo, Stenetorp, Narang, and Hupkes]{madaan2024quantifying}
L.~Madaan, A.~K. Singh, R.~Schaeffer, A.~Poulton, S.~Koyejo, P.~Stenetorp, S.~Narang, and D.~Hupkes.
\newblock Quantifying variance in evaluation benchmarks.
\newblock \emph{arXiv preprint arXiv:2406.10229}, 2024.
\newblock URL \url{https://arxiv.org/abs/2406.10229}.

\bibitem[Marx et~al.(2020)Marx, Calmon, and Ustun]{marx2020predictive}
C.~Marx, F.~Calmon, and B.~Ustun.
\newblock Predictive multiplicity in classification.
\newblock In \emph{International Conference on Machine Learning}, 2020.

\bibitem[McLeish et~al.(2025)McLeish, Kirchenbauer, Miller, Singh, Bhatele, Goldblum, Panda, and Goldstein]{gadre2025gemstones}
S.~McLeish, J.~Kirchenbauer, D.~Y. Miller, S.~Singh, A.~Bhatele, M.~Goldblum, A.~Panda, and T.~Goldstein.
\newblock Gemstones: A model suite for multi-faceted scaling laws.
\newblock \emph{arXiv preprint arXiv:2502.06857}, 2025.
\newblock URL \url{https://arxiv.org/abs/2502.06857}.

\bibitem[Porian et~al.(2024)Porian, Wortsman, Jitsev, Schmidt, and Carmon]{porian2024resolving}
T.~Porian, M.~Wortsman, J.~Jitsev, L.~Schmidt, and Y.~Carmon.
\newblock Resolving discrepancies in compute-optimal scaling of language models.
\newblock \emph{arXiv preprint arXiv:2406.19146}, 2024.
\newblock URL \url{https://arxiv.org/abs/2406.19146}.

\bibitem[Ramapuram et~al.(2026)Ramapuram, Dhekane, Shidani, Busbridge, Mazoure, Gu, Webb, Likhomanenko, and Jaitly]{ramapuram2026scalingpropertiescontinuousdiffusion}
J.~Ramapuram, E.~G. Dhekane, A.~Shidani, D.~Busbridge, B.~Mazoure, Z.~Gu, R.~Webb, T.~Likhomanenko, and N.~Jaitly.
\newblock Scaling properties of continuous diffusion spoken language models, 2026.
\newblock URL \url{https://arxiv.org/abs/2604.24416}.

\bibitem[Rudin et~al.(2024)Rudin, Zhong, Semenova, Seltzer, Parr, Liu, Katta, Donnelly, Chen, and Boner]{rudin2024amazing}
C.~Rudin, C.~Zhong, L.~Semenova, M.~Seltzer, R.~Parr, J.~Liu, S.~Katta, J.~Donnelly, H.~Chen, and Z.~Boner.
\newblock Position: Amazing things come from having many good models.
\newblock In \emph{International Conference on Machine Learning}, 2024.

\bibitem[Schaeffer et~al.(2025)Schaeffer, Levi, Kirsch, Guenais, Miranda, Obbad, and Koyejo]{schaeffer2025evaluating}
R.~Schaeffer, N.~Levi, A.~Kirsch, T.~Guenais, B.~Miranda, E.~Obbad, and S.~Koyejo.
\newblock Evaluating the robustness of chinchilla compute-optimal scaling.
\newblock \emph{arXiv preprint arXiv:2509.23963}, 2025.
\newblock URL \url{https://arxiv.org/abs/2509.23963}.

\bibitem[Shukor et~al.(2025)Shukor, B{\'e}thune, Busbridge, Grangier, Fini, El-Nouby, and Ablin]{shukor2025scaling}
M.~Shukor, L.~B{\'e}thune, D.~Busbridge, D.~Grangier, E.~Fini, A.~El-Nouby, and P.~Ablin.
\newblock Scaling laws for optimal data mixtures.
\newblock \emph{arXiv preprint arXiv:2507.09404}, 2025.
\newblock URL \url{https://arxiv.org/abs/2507.09404}.

\bibitem[Videau et~al.(2026)Videau, Youbi-Idrissi, Lopez-Paz, and Ahuja]{videau2026skalingchinchillasexponentsmeet}
M.~Videau, B.~Youbi-Idrissi, D.~Lopez-Paz, and K.~Ahuja.
\newblock Skaling: Chinchilla's exponents meet kaplan's coupling, 2026.
\newblock URL \url{https://arxiv.org/abs/2608.07222}.

\bibitem[Virtanen et~al.(2020)Virtanen, Gommers, Oliphant, Haberland, Reddy, Cournapeau, Burovski, Peterson, Weckesser, Bright, {van der Walt}, Brett, Wilson, Millman, Mayorov, Nelson, Jones, Kern, Larson, Carey, Polat, Feng, Moore, {VanderPlas}, Laxalde, Perktold, Cimrman, Henriksen, Quintero, Harris, Archibald, Ribeiro, Pedregosa, {van Mulbregt}, and {SciPy 1.0 Contributors}]{2020SciPy}
P.~Virtanen, R.~Gommers, T.~E. Oliphant, M.~Haberland, T.~Reddy, D.~Cournapeau, E.~Burovski, P.~Peterson, W.~Weckesser, J.~Bright, S.~J. {van der Walt}, M.~Brett, J.~Wilson, K.~J. Millman, N.~Mayorov, A.~R.~J. Nelson, E.~Jones, R.~Kern, E.~Larson, C.~J. Carey, {\.I}.~Polat, Y.~Feng, E.~W. Moore, J.~{VanderPlas}, D.~Laxalde, J.~Perktold, R.~Cimrman, I.~Henriksen, E.~A. Quintero, C.~R. Harris, A.~M. Archibald, A.~H. Ribeiro, F.~Pedregosa, P.~{van Mulbregt}, and {SciPy 1.0 Contributors}.
\newblock {{SciPy} 1.0: Fundamental Algorithms for Scientific Computing in Python}.
\newblock \emph{Nature Methods}, 17:\penalty0 261--272, 2020.
\newblock \doi{10.1038/s41592-019-0686-2}.

\bibitem[Watson-Daniels et~al.(2023)Watson-Daniels, Parkes, and Ustun]{watsondaniels2023predictive}
J.~Watson-Daniels, D.~C. Parkes, and B.~Ustun.
\newblock Predictive multiplicity in probabilistic classification.
\newblock In \emph{Proceedings of the AAAI Conference on Artificial Intelligence}, 2023.

\bibitem[Xin et~al.(2022)Xin, Zhong, Chen, Takagi, Seltzer, and Rudin]{xin2022exploring}
R.~Xin, C.~Zhong, Z.~Chen, T.~Takagi, M.~I. Seltzer, and C.~Rudin.
\newblock Exploring the whole {R}ashomon set of sparse decision trees.
\newblock In \emph{Advances in Neural Information Processing Systems}, 2022.

\end{thebibliography}

\appendix
\crefalias{section}{appendix}
\crefalias{subsection}{appendix}
\section{Proofs for Certified Fitting Algorithms}
\label{app:proofs}

This appendix gives the complete algorithms and proves
\Cref{thm:main}.

We use the box-corner notation $B^{\lowercorner},B^{\uppercorner}$ introduced in \Cref{sec:branch-and-bound}. The proofs work in the reparameterized parameter space on the compact domain $\mathcal U$, with $\Theta=\phi^{-1}(\mathcal U)$ and $\mathcal S^\phi_\epsilon=\phi(\mathcal S_\epsilon)$ as defined in \Cref{eq:original-search-domain,eq:transformed-sublevel-set}.

\begin{lemma}[Lipschitz continuity of predictions]
\label{lem:predictions-lipschitz}
Let \(\mathcal U\) be the compact hyperrectangle in \Cref{eq:certified-search-domain}, and write \(u=(u_1,\ldots,u_d)\), where \(d=2k+1\). For every fixed \(x\in\mathbb R_{>0}^k\), there exist finite constants \(K_1(x),\ldots,K_d(x)\) such that, for all \(u,v\in\mathcal U\),
\begin{equation}
    \left|L(x;u)-L(x;v)\right|
    \leq
    \sum_{j=1}^d K_j(x)\lvert u_j-v_j\rvert .
    \label{eq:prediction-lipschitz}
\end{equation}
Consequently, for every box
\(B=[B^{\lowercorner},B^{\uppercorner}]\subseteq\mathcal U\),
\begin{equation}
    \sup_{u\in B}L(x;u)-\inf_{u\in B}L(x;u)
    \leq
    \sum_{j=1}^d
        K_j(x)
        \bigl(B^{\uppercorner}_j-B^{\lowercorner}_j\bigr).
    \label{eq:prediction-box-width}
\end{equation}
\end{lemma}

Note that an immediate consequence is that the conclusion holds for log-predictions as well. Indeed, \Cref{eq:certified-search-domain} gives
\(L(x;u)\geq E\geq\underline E>0\), so the mean value theorem implies
\begin{equation}
    \left|\log L(x;u)-\log L(x;v)\right|
    \leq
    \frac{1}{\underline E}
    \left|L(x;u)-L(x;v)\right|
    \leq
    \frac{1}{\underline E}
    \sum_{j=1}^d K_j(x)\lvert u_j-v_j\rvert .
    \label{eq:log-prediction-lipschitz}
\end{equation}
In particular, for every training input \(x^{(n)}\) and box
\(B=[B^{\lowercorner},B^{\uppercorner}]\), \Cref{lem:monotonicity} gives
\begin{equation}
    0\leq p_n^+(B)-p_n^-(B)
    \leq
    \sum_{j=1}^d \frac{ K_j(x^{(n)}) }{\underline{E}}
    \bigl(B^{\uppercorner}_j-B^{\lowercorner}_j\bigr).
    \label{eq:training-log-prediction-box-width}
\end{equation}

\begin{proof}
Fix \(x\in\mathbb R_{>0}^k\), and abbreviate
\(s_i=s_i(x)\). By \Cref{eq:monotone-scaling-law}, we can write \(
    L(x;u)=E+\sum_{i=1}^k C_i r_i^{s_i}
\), whose partial derivatives are
\begin{align}
    \frac{\partial L(x;u)}{\partial E}=1,\qquad
    \frac{\partial L(x;u)}{\partial C_i}=r_i^{s_i},\qquad
    \frac{\partial L(x;u)}{\partial r_i} =C_i s_i r_i^{s_i-1}.
        \label{eq:lipschitz-lemma-prediction-gradients}
\end{align}
Because \(\mathcal U\) is compact and \(r_i\geq\underline r_i > 0\), these derivatives are continuous and bounded on \(\mathcal U\). Hence
\[
    K_j(x)
    :=
    \sup_{u\in\mathcal U}
    \left|
        \frac{\partial L(x;u)}{\partial u_j}
    \right|
    <\infty .
\]
Noting that $L(x; u)$ is clearly continuously differentiable on the entire domain $\mathcal{U}$, these derivative bounds yield the desired Lipschitz estimate, as follows.

For \(u,v\in\mathcal U\), the line segment \(w(t)=v+t(u-v)\), \(t\in[0,1]\), remains in \(\mathcal U\). The fundamental theorem of calculus therefore gives
\begin{align*}
    \left|L(x;u)-L(x;v)\right|
    &=
    \left|
        \int_0^1
        \nabla_u L(x;w(t))^\top(u-v)\,dt
    \right| \\
    &\leq
    \sum_{j=1}^d K_j(x)\lvert u_j-v_j\rvert .
\end{align*}
This proves \Cref{eq:prediction-lipschitz}. Taking the supremum over \(u,v\in B\) proves \Cref{eq:prediction-box-width}.
\end{proof}

\begin{lemma}[Valid and consistent loss enclosures]
\label{lem:consistent-loss-enclosure}
Let \(\mathcal U\) be the compact parameter domain in \Cref{eq:certified-search-domain}, and let \(d=2k+1\). For a box \(B=[B^{\lowercorner},B^{\uppercorner}]\subseteq\mathcal U\), recall that
\[
    p_n^-(B)
    =
    \log L(x^{(n)};B^{\lowercorner}),
    \qquad
    p_n^+(B)
    =
    \log L(x^{(n)};B^{\uppercorner}),
    \qquad
    t_n=\log y^{(n)}.
\]
Let \(\underline L_{\texttt{fit}}(B),\overline L_{\texttt{fit}}(B) \) be the lower and upper bounds defined in
\Cref{eq:box-loss-lower-bound,eq:box-loss-upper-bound}, respectively.
Then, the following bounds hold:
\begin{equation}
    \underline L_{\texttt{fit}}(B)
    \leq
    \inf_{u\in B}L_{\texttt{fit}}(u)
    \leq
    \sup_{u\in B}L_{\texttt{fit}}(u)
    \leq
    \overline L_{\texttt{fit}}(B).
    \label{eq:box-loss-enclosure}
\end{equation}

Moreover, the bounds are consistent as the box side lengths tend to zero, in the following sense. Let \(K_j(x^{(n)})\) be the constants from \Cref{lem:predictions-lipschitz}. Then
\begin{align}
    0
    &\leq
    \overline L_{\texttt{fit}}(B)
    -
    \underline L_{\texttt{fit}}(B)
     \leq
    \frac{\delta}{N_{\mathrm{obs}}\underline E}
    \sum_{n=1}^{N_{\mathrm{obs}}}\sum_{j=1}^d
        K_j(x^{(n)})
        \bigl(
            B^{\uppercorner}_j-B^{\lowercorner}_j
        \bigr),
    \label{eq:box-loss-enclosure-width}
\end{align}
where $\delta$ is the Huber loss parameter.
\end{lemma}

\begin{proof}
Fix a box \(B\). By \Cref{lem:monotonicity}, for every \(u\in B\) and every
observation \(n\), \(\log L(x^{(n)};u)\in[p_n^-(B),p_n^+(B)]\), so the
corresponding residual lies in
\[
    R_n(B)=[p_n^-(B)-t_n,\;p_n^+(B)-t_n].
\]

Because \(\operatorname{Huber}_\delta\) is even and nondecreasing in the
residual magnitude, its minimum over the compact interval \(R_n(B)\) is
attained at the point closest to \(0\) and its maximum at whichever endpoint
is farthest from \(0\):
\begin{align*}
    \min_{z\in R_n(B)}\operatorname{Huber}_{\delta}(z)
    &=
    \operatorname{Huber}_{\delta}\bigl(\operatorname{dist}(0,R_n(B))\bigr)
    =
    \operatorname{Huber}_{\delta}\bigl(\operatorname{dist}(t_n,I_n(B))\bigr)
    =
    \operatorname{Huber}_{\delta}\bigl(d_n(B)\bigr),\\
    \max_{z\in R_n(B)}\operatorname{Huber}_{\delta}(z)
    &=
    \max\left\{
        \operatorname{Huber}_{\delta}\bigl(p_n^-(B)-t_n\bigr),
        \operatorname{Huber}_{\delta}\bigl(p_n^+(B)-t_n\bigr)
    \right\},
\end{align*}
where the last equality uses the definition
\(d_n(B)=\operatorname{dist}(t_n,I_n(B))\). These are exactly the
per-observation summands of \(\underline L_{\texttt{fit}}(B)\) and
\(\overline L_{\texttt{fit}}(B)\) in
\Cref{eq:box-loss-lower-bound,eq:box-loss-upper-bound}.

Hence, for every \(u\in B\), since the residual
\(\log L(x^{(n)};u)-t_n\in R_n(B)\),
\[
    \operatorname{Huber}_{\delta}\bigl(d_n(B)\bigr)
    \leq
    \operatorname{Huber}_{\delta}\bigl(\log L(x^{(n)};u)-t_n\bigr)
    \leq
    \max_{z\in R_n(B)}\operatorname{Huber}_{\delta}(z).
\]
Averaging over \(n\) and noting that the outer bounds are independent of
\(u\), we take the infimum and supremum over \(u\in B\) to obtain
\Cref{eq:box-loss-enclosure}.

It remains to bound the enclosure width. By construction,
\(\overline L_{\texttt{fit}}(B)\geq\underline L_{\texttt{fit}}(B)\).
For the upper direction, \(\operatorname{Huber}_{\delta}\) is globally
\(\delta\)-Lipschitz, so for every \(n\),
\[
    \max_{z\in R_n(B)}\operatorname{Huber}_{\delta}(z)
    -
    \min_{z\in R_n(B)}\operatorname{Huber}_{\delta}(z)
    \leq
    \delta\,\bigl|R_n(B)\bigr|
    =
    \delta\bigl(p_n^+(B)-p_n^-(B)\bigr),
\]
using \(|R_n(B)|=p_n^+(B)-p_n^-(B)\). Combining this with the definitions
\Cref{eq:box-loss-lower-bound,eq:box-loss-upper-bound} and the identity
\(\min_{z\in R_n(B)}\operatorname{Huber}_{\delta}(z)
=\operatorname{Huber}_{\delta}(d_n(B))\) above,
\begin{align*}
    \overline L_{\texttt{fit}}(B)-\underline L_{\texttt{fit}}(B)
    &=
    \frac{1}{N_{\mathrm{obs}}}\sum_{n=1}^{N_{\mathrm{obs}}}
    \left[
        \max_{z\in R_n(B)}\operatorname{Huber}_{\delta}(z)
        -
        \operatorname{Huber}_{\delta}\bigl(d_n(B)\bigr)
    \right]\\
    &\leq
    \frac{\delta}{N_{\mathrm{obs}}}\sum_{n=1}^{N_{\mathrm{obs}}}\bigl(p_n^+(B)-p_n^-(B)\bigr).
\end{align*}
Finally, applying the log-prediction width bound
\Cref{eq:training-log-prediction-box-width} to every observation yields
\Cref{eq:box-loss-enclosure-width}.
\end{proof}

Lastly, we record the following fact about the side lengths of boxes obtained by splitting with \(\textsc{BranchRegion}\):

\begin{remark}[Bisection box size]
Consider a descendant box \(B_m\) obtained after \(m\) successive applications of \(\textsc{BranchRegion}\) along one branch of the search tree, and write $B^{\lowercorner}_{m,j},B^{\uppercorner}_{m,j}$ for the $j$-th coordinates of its corners. Let \(d=2k+1\) and \(W\coloneqq\max_j(\mathcal U^{\uppercorner}_j-\mathcal U^{\lowercorner}_j)\). Then
\begin{equation}
    B^{\uppercorner}_{m,j}-B^{\lowercorner}_{m,j}
    \leq
    W\,2^{-\lfloor m/d\rfloor}
    \rightarrow 0
    \qquad
    \text{for every }j.
    \label{eq:balanced-bisection-width}
\end{equation}
To see this for the widest-side rule, let \(w\) be the largest side length of some box on the branch. Side lengths never increase along the branch. While some side is longer than \(w/2\), the widest side is one of these sides, and bisecting it leaves it no longer than \(w/2\). There are at most \(d\) such sides, so after at most \(d\) further splits every side has length at most \(w/2\). Induction on \(m\) gives \Cref{eq:balanced-bisection-width}. The rule used in \Cref{alg:viable-set-cover} is the widest-side rule applied to the normalized widths \((B^{\uppercorner}_j-B^{\lowercorner}_j)/(\mathcal U^{\uppercorner}_j-\mathcal U^{\lowercorner}_j)\), so the same argument gives \(B^{\uppercorner}_{m,j}-B^{\lowercorner}_{m,j}\leq(\mathcal U^{\uppercorner}_j-\mathcal U^{\lowercorner}_j)2^{-\lfloor m/d\rfloor}\), which implies \Cref{eq:balanced-bisection-width}.
In other words, every box split indefinitely shrinks to zero width in every coordinate.
\end{remark}

\subsection{Proof of certified fitting}

\begin{proof}[Proof of \Cref{thm:main}\textup{(a)}]
Because \(\mathcal U\) is compact and
\(L_{\texttt{fit}}\) is continuous on \(\mathcal U\), the minimum
\[
    L_{\texttt{fit}}^\star
    :=
    \min_{u\in\mathcal U} L_{\texttt{fit}}(u)
\]
is attained for some $u^\ast \in \mathcal{U}$ by the Extreme Value Theorem.

We first establish the certificate maintained by \Cref{alg:certified-fit}.
Namely, we first show that at every iteration, the incumbent $\hat{u}$ provides an upper bound on the globally optimal loss, and that a corresponding global lower bound is obtained from the smallest lower bound among the active boxes. In other words, we will show that the algorithm maintains an interval $[\min_{B \in \mathcal{Q}} \underline{L}_\texttt{fit}(B), L_\texttt{fit}(\hat{u})]$ that is guaranteed to contain the globally optimal loss at every iteration; termination occurs once the width of this interval is at most
\(\epsilon\).

Let \(\mathcal Q_t\) be the collection of active boxes and let \(\hat u_t\) be the incumbent at the beginning of iteration \(t\). Since the incumbent is updated only when a better tested point is found, \(L_{\texttt{fit}}(\hat u_t)\) is nonincreasing in \(t\).

At every iteration $t$, the active boxes $\mathcal{Q}_t$ together with the boxes discarded at earlier iterations $\mathcal{D}_t$ cover \(\mathcal U\). If a box \(D\) was discarded at iteration \(s\), then the pruning rule and
\Cref{lem:consistent-loss-enclosure} give
\[
    L_{\texttt{fit}}(u)
    \geq
    \underline L_{\texttt{fit}}(D)
    \geq
    L_{\texttt{fit}}(\hat u_s)
    \geq
    L_{\texttt{fit}}(\hat u_t)
    \qquad
    \text{for every \(u\in D\) and \(t\geq s\)}.
\]
Thus a discarded box cannot contain a point that improves upon the
current incumbent.

When \(\mathcal Q_t\neq\varnothing\), define
\[
    \underline L_t
    :=
    \min_{B\in\mathcal Q_t}
    \underline L_{\texttt{fit}}(B).
\]
Every point in an active box has loss at least \(\underline L_t\),
while every point in a discarded box has loss at least
\(L_{\texttt{fit}}(\hat u_t)\). Consequently,
\begin{equation}
    \min\left\{
        L_{\texttt{fit}}(\hat u_t),
        \underline L_t
    \right\}
    \leq
    L_{\texttt{fit}}^\star
    \leq
    L_{\texttt{fit}}(\hat u_t).
    \label{eq:certified-fit-invariant}
\end{equation}
If the stopping condition
\(
    L_{\texttt{fit}}(\hat u_t)
    \leq
    \underline L_t+\epsilon
\)
holds, then \Cref{eq:certified-fit-invariant} immediately gives
\[
    0
    \leq
    L_{\texttt{fit}}(\hat u_t)-L_{\texttt{fit}}^\star
    \leq
    \epsilon,
\]
as desired.
On the other hand, if \(\mathcal Q_t\) is empty, the
discarded boxes cover \(\mathcal U\), and
\(
    L_{\texttt{fit}}(u)
    \geq
    L_{\texttt{fit}}(\hat u_t)
\) for every $u\in\mathcal U$.
Thus \(L_{\texttt{fit}}(\hat u_t)=L_{\texttt{fit}}^\star\).
Hence, if the algorithm terminates, $L_\texttt{fit} (\hat{u}) - L^\star_\texttt{fit} \leq \epsilon$ is guaranteed to hold for the returned candidate $\hat{u}$. The original parameters $\hat\theta=\phi^{-1}(\hat u)$ therefore belong to $\mathcal S_\epsilon$.

It remains only to show that this stopping condition is reached after finitely many iterations.

We can think of the boxes generated by \(\textsc{BranchRegion}\) as nodes of a binary subdivision tree, with \(\mathcal U\) as the root at depth zero, and the boxes $\{B_1, B_2\}$ obtained from \(\textsc{BranchRegion}(B)\) at some iteration as the children of $B$. All boxes at the same depth $m$ have been subdivided (branched) $m$ times, and so we can use the loss-enclosure bound \Cref{eq:box-loss-enclosure-width} with the bisection estimate \Cref{eq:balanced-bisection-width}, to bound the tightness of the upper and lower bounds for those boxes:
\begin{equation}
    \overline L_{\texttt{fit}}(B)
    -
    \underline L_{\texttt{fit}}(B)
    \leq
    K_{\mathrm{fit}}\,
    2^{-\lfloor m/d\rfloor},
    \label{eq:loss-gap-by-tree-depth}
\end{equation}
where \(d=2k+1\) and
\(
    K_{\mathrm{fit}}
    :=
    \frac{\delta\,W}{N_{\mathrm{obs}}\underline E}
    \sum_{n=1}^{N_{\mathrm{obs}}}\sum_{j=1}^d
        K_j(x^{(n)})
    <\infty
\).

Choose \(M<\infty\) such that
\[
    K_{\mathrm{fit}}2^{-\lfloor M/d\rfloor}\leq\epsilon.
\]
We will show, by contradiction, that the algorithm never splits a box of depth at least \(M\). Otherwise, let \(B\) be such a box selected for splitting at iteration $t$. Since \(B\) has minimum priority and its candidate \(u_B\in B\) was evaluated when
\(B\) was created,
\[
    L_{\texttt{fit}}(\hat u_t)
    \leq L_{\texttt{fit}}(u_B)
    \leq \overline L_{\texttt{fit}}(B)
    \leq \underline L_{\texttt{fit}}(B)+\epsilon
    = \underline L_t+\epsilon.
\]
Thus, the stopping condition would hold, contradicting that \(B\) was
selected for splitting.

Therefore, the algorithm splits only boxes of depth less than \(M\).
Since the binary subdivision tree has only finitely many such boxes
(at most \(2^M-1\)), and the algorithm adds two child nodes to the tree at every iteration, the algorithm terminates after finitely many
iterations.
    

\end{proof}


\begin{remark}[Iteration complexity in $\epsilon$]
\label{remark:certified-fitting-complexity}
The preceding argument also gives an explicit worst-case bound on the number of branching iterations. Let
\[
    q
    =
    \left\lceil
        \log_2
        \max\left\{
            1,
            \frac{K_{\mathrm{fit}}}{\epsilon}
        \right\}
    \right\rceil,
    \qquad
    M=dq.
\]
Then the algorithm splits only boxes of depth less than \(M\), and therefore performs at most
\[
    2^M-1
    =
    2^{dq}-1
\]
branching iterations. In particular, this is at most
\[
    \left(\frac{2K_{\mathrm{fit}}}{\epsilon}\right)^d.
\]
This gives us $\mathcal{O}(\epsilon^{-d})$ iteration complexity in $\epsilon$ as $\epsilon \to 0$. 

The queue contains at most one more box than the number of branching iterations, so the same bound in $\epsilon$ applies to its memory usage. This bound is very conservative: it assumes that the complete subdivision tree is explored and does not account for pruning.
\end{remark}

\begin{remark}[Iteration complexity in $d$]
\label{remark:certified-fitting-complexity-in-dimension}
We can also bound the iteration complexity in terms of the dimension $d=2k+1$, which enters through $K_{\mathrm{fit}}$ as well as through the exponent. By the gradient expressions in \Cref{eq:lipschitz-lemma-prediction-gradients}, $K_{\mathrm{fit}}\le c\,d$ for a constant $c$ independent of $d$, for any family of problems in which the parameter domains $\mathcal U$ are uniformly bounded coordinatewise (and uniformly bounded away from zero for $E$ and $r_i$) and the inputs $x_i^{(n)}$ are uniformly bounded and uniformly bounded away from $0$. Under these conditions, the worst-case branching-iteration and queue-memory complexities are at most $(2c\,d/\epsilon)^d$.

\end{remark}


\begin{proof}[Proof of \Cref{thm:main}\textup{(b)}]
Write
\[
    L_{\texttt{fit}}^\star
    =
    \min_{u\in\mathcal U}L_{\texttt{fit}}(u).
\]
By \Cref{thm:main}\textup{(a)}, Phase~1 of
\Cref{alg:viable-set-cover} terminates and returns a candidate
\(\hat u\) and a certified lower bound \(\underline L\) satisfying
\begin{equation}
    \underline L
    \leq
    L_{\texttt{fit}}^\star
    \leq
    L_{\texttt{fit}}(\hat u)
    \leq
    \underline L+\frac{\epsilon}{3}.
    \label{eq:viable-search-initial-certificate}
\end{equation}

Two invariants hold at every iteration of the algorithm:
\begin{itemize}
    \item Every box $B\in\mathcal B$ is contained in the sublevel set $\mathcal S^\phi_\epsilon$.
    \item Boxes discarded are strictly disjoint from the $\epsilon / 2$ sublevel set $\mathcal{S}^\phi_{\epsilon / 2}$.
\end{itemize}
Both follow directly from \Cref{lem:consistent-loss-enclosure} and \Cref{eq:viable-search-initial-certificate}.

It remains to show that the subdivision procedure terminates.
Consider again the subdivision history as a rooted binary tree: the root is \(\mathcal U\), and whenever a box is subdivided, its two children are the resulting subboxes (in the same way as in the proof of \Cref{thm:main}\textup{(a)}). Thus, every box ever accepted, discarded, or placed in the queue is a node of this tree. Note that a leaf in the tree can only ever be ``iterated upon'' once, in which case it is either expanded into two children, accepted, or discarded and then never iterated upon again.

As in the proof of \Cref{thm:main}\textup{(a)}, \Cref{eq:loss-gap-by-tree-depth} shows that a box \(B\) at depth \(m\) in the tree satisfies
\[
    \overline L_{\texttt{fit}}(B)
    -
    \underline L_{\texttt{fit}}(B)
    \leq
    K_{\mathrm{fit}}2^{-\lfloor m/d\rfloor}.
\]
On the other hand, a box is subdivided only if it satisfies neither classification rule. For any such box,
\(
    \underline L_{\texttt{fit}}(B)
    \leq L_{\texttt{fit}}(\hat u)+\frac{\epsilon}{2}\), \(\overline L_{\texttt{fit}}(B) > \underline L+\epsilon\), which, together with \Cref{eq:viable-search-initial-certificate}, implies
\[
    \overline L_{\texttt{fit}}(B)
    -
    \underline L_{\texttt{fit}}(B)
    >
    \frac{\epsilon}{2}
    -
    \bigl(L_{\texttt{fit}}(\hat u)-\underline L\bigr)
    \geq
    \frac{\epsilon}{6}.
\]
Let $M=d\lceil\log_2\max\{1,6K_{\mathrm{fit}}/\epsilon\}\rceil$, so that $K_{\mathrm{fit}}2^{-\lfloor M/d\rfloor}\le\epsilon/6$. A box at depth $M$ or greater then has loss gap at most $\epsilon/6$ and is never subdivided, so the subdivision tree has depth at most $M$ and at most $2^{M+1}-1$ nodes. Since the algorithm processes a distinct node at every iteration, Phase~2 terminates after at most $2^{M+1}-1\le2\max\{(12K_{\mathrm{fit}}/\epsilon)^d,1\}$ iterations, which is $\mathcal O(\epsilon^{-d})$ for fixed $d$.

At termination, the accepted and discarded boxes cover $\mathcal U$. The invariants therefore give $\mathcal S^\phi_{\epsilon/2}\subseteq\bigcup_{B\in\mathcal B}B\subseteq\mathcal S^\phi_\epsilon$. Applying $\phi^{-1}$ yields $\mathcal S_{\epsilon/2}\subseteq\mathcal C\subseteq\mathcal S_\epsilon$, proving the original-coordinate statement.
\end{proof}

\begin{proof}[Proof of \Cref{thm:main}\textup{(c)}]
We prove the upper-endpoint guarantee; the lower endpoint follows
symmetrically by applying a mirror argument to \(-L(x^\ast;u)\).

By \Cref{thm:main}\textup{(a)}, Phase~1 terminates and returns bounds
satisfying
\begin{equation}
    \underline L
    \leq
    L_{\texttt{fit}}^\star
    \leq
    L_{\mathrm{fit}}^{\mathrm{ub}}
    \leq
    \underline L+\frac{\eta}{2}, \qquad  L_{\texttt{fit}}^\star \coloneqq \min_{u\in\mathcal U}L_{\texttt{fit}}(u).
    \label{eq:prediction-search-initial-certificate}
\end{equation}

First, we establish the invariant that throughout Phase~2, the boxes in \(\mathcal Q\) cover \(\mathcal S^\phi_\epsilon\). This holds initially because
\(\mathcal Q=\{\mathcal U\}\). When a child \(C\) is omitted from the
queue, the pruning rule and \Cref{lem:consistent-loss-enclosure} give
\[
    L_{\texttt{fit}}(u)
    \geq
    \underline L_{\texttt{fit}}(C)
    >
    L_{\mathrm{fit}}^{\mathrm{ub}}+\epsilon
    \geq
    L_{\texttt{fit}}^\star+\epsilon
\]
for every \(u\in C\). Hence \(C\cap\mathcal S^\phi_\epsilon=\varnothing\),
so pruning preserves the covering property.

With this, we can establish the correctness of the termination condition. Suppose that the algorithm returns while processing \(B\).

Since \(B\) has maximum priority in the queue $\mathcal{Q}$, and the boxes in the queue cover $\mathcal{S}^\phi_\epsilon$, \Cref{lem:monotonicity} implies that
\[
    \sup\mathcal L_{x^\ast,\epsilon}
    \leq
    L(x^\ast;B^{\uppercorner}).
\]
On the other hand, the return condition gives, for every \(u\in B\),
\[
    L_{\texttt{fit}}(u)
    \leq
    \overline L_{\texttt{fit}}(B)
    \leq
    \underline L+\epsilon+\eta
    \leq
    L_{\texttt{fit}}^\star+\epsilon+\eta.
\]
Thus \(B\subseteq\mathcal S^\phi_{\epsilon+\eta}\), and in particular
\[
    L(x^\ast;B^{\uppercorner})
    \leq
    \sup\mathcal L_{x^\ast,\epsilon+\eta}.
\]
Therefore, the returned value satisfies the claimed inequality:
\[
    \sup\mathcal L_{x^\ast,\epsilon}
    \leq
    \overline L_{x^\ast}
    \leq
    \sup\mathcal L_{x^\ast,\epsilon+\eta}.
\]

It remains to prove finite termination. We use an argument similar to those in the previous two proofs. Every box \(B\) inserted into
the queue satisfies
\[
    \underline L_{\texttt{fit}}(B)
    \leq
    L_{\mathrm{fit}}^{\mathrm{ub}}+\epsilon.
\]
If such a box is subdivided rather than returned, then
\[
\begin{aligned}
    \overline L_{\texttt{fit}}(B)
    -
    \underline L_{\texttt{fit}}(B)
    &>
    \underline L+\epsilon+\eta
    -
    \bigl(L_{\mathrm{fit}}^{\mathrm{ub}}+\epsilon\bigr) \\
    &=
    \eta-
    \bigl(L_{\mathrm{fit}}^{\mathrm{ub}}-\underline L\bigr)
    \geq
    \frac{\eta}{2},
\end{aligned}
\]
where the final inequality follows from
\Cref{eq:prediction-search-initial-certificate}.

Let \(M'=d\lceil\log_2\max\{1,2K_{\mathrm{fit}}/\eta\}\rceil\). By \Cref{eq:loss-gap-by-tree-depth}, every box of depth at least \(M'\) has loss gap at most \(\eta/2\), so no such box is subdivided, and the subdivision tree has at most \(2^{M'+1}-1\) nodes.
Since every iteration either returns or subdivides a distinct node of this tree, the number of iterations is finite. 
Moreover, the queue cannot become empty before the algorithm returns: a child containing a global minimizer always satisfies the insertion criterion.
Hence Phase~2 returns after finitely many iterations.
\end{proof}

\subsection{Prediction bounds at arbitrary target scales}
\label{app:arbitrary-target-prediction-bounds}

\Cref{alg:certified-prediction-bound} in the main text assumes that $x_i^\ast\geq x_i^{\min}$ for every input coordinate, in which case the upper corner of each parameter box maximizes the prediction at $x^\ast$. Here, we give an algorithm that removes this assumption.

Recall that the reparameterized scaling law is $L(x;u) = E+\sum_{i=1}^k C_i r_i^{s_i(x)}$, with $s_i(x)=\log\frac{x_i}{x_i^{\min}}$.
The prediction is always nondecreasing in $E$ and $C_i$.
Its monotonicity with respect to $r_i$ depends on the sign of
$s_i(x^\ast)$:
\begin{equation}
    \frac{\partial L(x^\ast;u)}{\partial r_i}
    =
    C_i s_i(x^\ast)r_i^{s_i(x^\ast)-1}.
\end{equation}
Thus, when $x_i^\ast<x_i^{\min}$, the prediction is nonincreasing
rather than nondecreasing in $r_i$.

For a box $B=[B^{\lowercorner},B^{\uppercorner}]$, define
$\textsc{MaxCorner}(B,x^\ast)$ by selecting the upper endpoint of
the $E$ and $C_i$ coordinates and selecting
\begin{equation}
    r_i\bigl(\textsc{MaxCorner}(B,x^\ast)\bigr)
    =
    \begin{cases}
        r_i(B^{\uppercorner}),
            & x_i^\ast\geq x_i^{\min},\\
        r_i(B^{\lowercorner}),
            & x_i^\ast<x_i^{\min}.
    \end{cases}
\end{equation}
This corner satisfies
\begin{equation}
    L(x^\ast;u)
    \leq
    L\bigl(x^\ast;\textsc{MaxCorner}(B,x^\ast)\bigr)
    \qquad
    \text{for every }u\in B.
\end{equation}

The general prediction-bound algorithm therefore differs from the main-text version only by replacing each upper prediction corner $B^{\uppercorner}$ with $\textsc{MaxCorner}(B,x^\ast)$, the corner that maximizes the prediction at $x^\ast$ in the general case.
The fitting loss bounds are otherwise unchanged. The algorithm is given in \Cref{alg:certified-prediction-bound-general}.

\begin{algorithm}
\SetAlgoLined
\KwIn{
Search domain $\mathcal U=[\mathcal U^{\lowercorner}, \mathcal U^{\uppercorner}]$; dataset $\mathcal D$; target scale $x^\ast$;
inner loss tolerance $\epsilon>0$; loss tolerance slack $\eta>0$;
}
\KwOut{Prediction bound $\overline L_{x^\ast}$ satisfying
\[
    \sup\mathcal L_{x^\ast,\epsilon}
    \leq
    \overline L_{x^\ast}
    \leq
    \sup\mathcal L_{x^\ast,\epsilon + \eta} \qquad
    \text{where } \mathcal L_{x^\ast,\epsilon} \coloneqq \{L(x^\ast; u): u \in \mathcal{S}^\phi_{\epsilon}\}.
\]
}

\tcp{Phase 1: certify the globally optimal fitting loss}
$(\hat u,\underline L) \gets \textsc{CertifiedFittingSearch}
(\mathcal U,\eta/2,\mathcal D)$\;
$L_{\mathrm{fit}}^{\mathrm{ub}}
\gets
L_{\texttt{fit}}(\hat u;\mathcal D)$
\tcp*[f]{We have
$\underline L
\leq L_{\texttt{fit}}^\star
\leq L_{\mathrm{fit}}^{\mathrm{ub}}
\leq \underline L+\eta/2$
}

\tcp{Phase 2: search for the largest viable prediction}
Initialize a max-priority queue $\mathcal Q$\;
$\mathcal Q.\textsc{Push}
\bigl(
    \mathcal U,
    \text{priority}=L\bigl(x^\ast;\textsc{MaxCorner}(\mathcal U,x^\ast)\bigr)
\bigr)$
\tcp*[f]{Largest prediction first}
\;

\While{$\mathcal Q$ is not empty}{
    $B\gets\mathcal Q.\textsc{PopMax}()$
    \tcp*[f]{Box with the largest remaining prediction upper bound}\;

    Compute $\underline L_{\texttt{fit}}(B)$
    and $\overline L_{\texttt{fit}}(B)$\;

    \If(\tcp*[f]{$B$ is entirely contained in
        $\mathcal S^\phi_{\epsilon+\eta}$}){
        $\overline L_{\texttt{fit}}(B)
        \leq
        \underline L+\epsilon+\eta$
    }{
        \KwRet $L\bigl(x^\ast;\textsc{MaxCorner}(B,x^\ast)\bigr)$\;
    }
    \ForEach{child box $C\in\textsc{BranchRegion}(B)$}{
        Compute $\underline L_{\texttt{fit}}(C)$\;

        \If{
            $\underline L_{\texttt{fit}}(C)
            \leq
            L_{\mathrm{fit}}^{\mathrm{ub}}+\epsilon$
        }{
            $\mathcal Q.\textsc{Push}
            \bigl(
                C,
                \text{priority}=L\bigl(x^\ast;\textsc{MaxCorner}(C,x^\ast)\bigr)
            \bigr)$\;
        }
    }
}
\caption{Certified prediction upper bound at an arbitrary target scale}
\label{alg:certified-prediction-bound-general}
\end{algorithm}


\section{Datasets}
\label{app:datasets}
\subsection{Literature Scaling Ladder Datasets}
\label{app:scaling-ladder-datasets}
We run the evaluations on four real-world scaling ladder datasets from the literature:

\textbf{Apple DCLM Downstream Metrics~\citep{apple2025revisiting,li2024datacomp}:}
    This dataset comprises pretraining validation loss and downstream evaluations across language models ranging from $17\text{M}$ to $17\text{B}$ parameters trained on the DataComp for Language Models (DCLM-Baseline) corpus~\citep{li2024datacomp}. 
    The training runs (approximately) span five discrete Token-to-Parameter Ratios (TPR, $D/N$): $\frac{D}{N} \in \{10.0, 20.0, 40.0, 80.0, 160.0\}$.\footnote{The true TPRs vary slightly around these values, because model depth must be an integer while preserving a fixed width-to-depth ratio.}
    We use all $149$ runs evaluated on the native training mixture, with training compute spanning $C \approx 6ND \in [1.05\times 10^{18},\, 3.83\times 10^{22}]$ FLOPs. All $149$ runs are used for 2D parametric fitting ($N, D \to L$).
    
\textbf{Delphi Open Scaling Suite~\citep{marin2026delphi}:}
    Delphi is an open scaling suite produced by the Marin Project, consisting of decoder-only language models utilizing the Qwen~3 architecture trained on a multi-domain corpus of Nemotron-CC, StarCoderData, and ProofPile~2. 
    The suite includes $88$ training runs: $81$ IsoFLOP sweep runs distributed across seven compute budgets from $3.28\times 10^{18}$ to $3.64\times 10^{20}$ FLOPs and $7$ large-scale held-out validation runs up to $1.0\times 10^{23}$ FLOPs, used to verify extrapolations up to $300\times$ the training compute.
    We strictly fit scaling laws on the $81$ IsoFLOP sweep runs ($C \le 3.64\times 10^{20}$ FLOPs), holding out the large-scale runs for out-of-distribution evaluation.
    
Because raw training loss logs are sometimes not published in numerical formats, some papers rely on digitized curves extracted from publication figures.
We source both the Chinchilla and Llama~3 scaling ladder data from the standardized ``Open Athena IsoFLOP Experiments'' collection~\citep{czech2026problems},\footnote{Dataset: \url{https://huggingface.co/datasets/open-athena/isoflop-experiments}; project page: \url{https://openathena.ai/scaling-law-analysis/}.} which contains digitized IsoFLOP experiment data from prominent scaling studies. We use the following two datasets:

 \textbf{Chinchilla MassiveText (Epoch AI Extraction)~\citep{besiroglu2024chinchilla,hoffmann2022training,czech2026problems}:}
    This dataset consists of $123$ training runs of standard Transformer models evaluated on MassiveText across nine IsoFLOP compute tiers ($C \in [6.00\times 10^{18},\, 3.00\times 10^{21}]$ FLOPs). 
    The underlying values were originally extracted and audited via SVG vector plot digitization by Epoch AI~\citep{besiroglu2024chinchilla} from the IsoFLOP figures in DeepMind's original Chinchilla study~\citep{hoffmann2022training}, and subsequently standardized into the Open Athena benchmark suite on which we relied.
    
 \textbf{Meta Llama 3 IsoFLOP Suite~\citep{dubey2024llama,czech2026problems}:}
    This set consists of $133$ training runs across ten IsoFLOP compute tiers spanning $C \in [6.00\times 10^{18},\, 1.00\times 10^{22}]$ FLOPs. 
    The training loss points were digitized by the Open Athena Foundation~\citep{czech2026problems} directly from the vector graphics of Meta's Llama~3 Technical Report~\citep{dubey2024llama}.

\section{Experimental Details}
\label{sec:experimental-details}

Below, we describe the experimental setup and the procedure for generating each of the figures in the main text.
Unless otherwise noted, all figures use the four empirical scaling ladder datasets described in \Cref{app:scaling-ladder-datasets} and the fitting loss and parametric family described in \Cref{sec:chinchilla-fitting}.
All certified branch-and-bound searches use the C++ implementation of the algorithms in \Cref{sec:main-result}, with the monotonic reparameterization from \Cref{sec:setup-monotonicity}.

\paragraph{Figure~\ref{fig:non-identifiability-scaling-law}: The Rashomon Effect in Scaling Law Fitting.}
This figure illustrates the Rashomon effect (near non-identifiability) on synthetic data.
The ground-truth scaling function is a continuous piecewise power law with two components: exponents $\alpha_1 \approx 0.47$ and $\alpha_2 \approx 0.33$ joined continuously at a transition midpoint $x_m \approx 1.7 \times 10^{18}$ FLOPs, designed so that multiple near-optimal scaling laws can fit it well:
\begin{equation*}
y_{\text{clean}}(C) = 
\begin{cases} 
2.4721 \left( \frac{C}{1.691 \times 10^{18}} \right)^{-0.4739} & \text{if } C \le 1.691 \times 10^{18} \\ 
2.4721 \left( \frac{C}{1.691 \times 10^{18}} \right)^{-0.3310} & \text{if } C > 1.691 \times 10^{18} 
\end{cases}
\end{equation*}
The training data were generated by sampling ten observations from this ground truth over $C \in [10^{17}, 10^{21}]$ FLOPs with 2\% multiplicative log-normal noise ($\sigma = 0.02$).
Two distinct L-BFGS fits of the standard three-parameter power law $L(C; \theta) = E + AC^{-\alpha}$ are obtained by initializing near each piecewise component's parameters: these achieve fitting losses differing by less than $0.12\%$, and the better of the two, $\ell^\star = 5.9208\times10^{-5}$, is certified by \Cref{alg:certified-fit} to be within $0.02\%$ of the optimal loss (the certified lower bound is $5.9198\times10^{-5}$), yet they predict a $74.6\times$ compute efficiency gap at $3 \times 10^{23}$ FLOPs ($300\times$ beyond the training range).

The certified optimum and the certified prediction band (yellow) are computed over the parameter box $0 < E \le 1.1\,y_{\min}$, $0 < A\,C_{\min}^{-\alpha} \le 1.3\,y_{\max}$, $0 \le \alpha \le 1.5\max(\alpha_1, \alpha_2)\approx0.71$, where $C_{\min} = 10^{17}$ is the smallest training compute and $y_{\min}, y_{\max}$ are the smallest and largest training losses.  The band uses the threshold $\tau = \ell^\star + \epsilon$ with $\epsilon = 1\%$ of $\ell^\star$ ($\epsilon = 5.92 \times 10^{-7}$). Since $\ell^\star$ is at least the optimal loss, every fit within $1\%$ of the optimum has loss at most $\tau$; by the certified lower bound, every fit with loss at most $\tau$ is within $1.02\%$ of the optimum.  
The band is computed at $49$ compute values from $10^{17}$ to $4.05\times10^{23}$, including $3\times10^{23}$, with a prediction-tolerance variant of \Cref{alg:certified-prediction-bound}: for each compute value and each direction, the search returns a certified bound on the predictions of all fits with loss at most $\tau$, and stops once this bound is within a factor $e^{\delta}$, $\delta = 5\times10^{-4}$, of the prediction of such a fit. At each of these compute values, every such prediction therefore lies inside the band, and each band edge is within $0.05\%$ of a prediction attained by such a fit; between them, the band is interpolated log-log. At $3\times10^{23}$ FLOPs the band is $[0.1202, 0.2392]$.  
Each search starts from a fit found by local search (SLSQP), which speeds up finding a fit near the edge of the range but does not affect the certificate. With an optimized build, each of the $98$ searches took at most $45$ seconds, and computing the full figure, including the certified fit, took about two minutes on $60$ cores.

\paragraph{Figure~\ref{fig:viable-extrap-fans-all}: Sublevel Prediction Envelopes Across Scaling Ladders.}
For each of the four empirical scaling ladder datasets (\Cref{app:scaling-ladder-datasets}), we fit the five-parameter Chinchilla scaling law $L(N,D;\theta) = E + A_1 N^{-\alpha_1} + A_2 D^{-\alpha_2}$ using Chinchilla Approach~3 (\Cref{sec:chinchilla-fitting}) at varying initialization-grid densities $K \in \{2, 5, 8\}$, each yielding $K^5$ initialization points over the standard heuristic parameter ranges.
The grid ranges are set to be $\log E \in [-1, 1]$, $\log A_1, \log A_2 \in [0, 25]$, and $\alpha_1, \alpha_2 \in [0, 2]$ in line with those in \citep{hoffmann2022training}. Along each coordinate, the $K$ grid points are placed at the midpoints of $K$ equal subintervals of the range, so grids of different densities are in general not nested.
At each grid density $K$, the best fit among all $K^5$ local optimizations with L-BFGS determines the fitting loss threshold $L_{\mathrm{chin},K}$.
To construct the viable prediction envelopes, we search for parameters $\theta$ satisfying $L_\texttt{fit}(\theta) \leq L_{\mathrm{chin},K}$ that maximize and minimize the predicted loss at $300\times$ the maximum training compute along the canonical slice with 20 tokens per parameter ($D = 20N$).
We use multi-start constrained optimization (SLSQP) from $20$ starting points: the Chinchilla fit, up to $9$ further distinct fits from the grid sweep (all densities) with fitting loss at most $1.05\,L_{\mathrm{chin},K}$, taken in order of increasing loss, and random perturbations of the Chinchilla fit.
For each grid density $K$, the shaded band lies between the full scaling curves of the two resulting extreme fits, and so shows predictions attained by fits that match or improve upon $L_{\mathrm{chin},K}$. Because the search is local, the band is an inner approximation of the full sublevel-set envelope.
Training data points are colored by their token-per-parameter ratio.

\paragraph{Figure~\ref{fig:certified-bb-subdivision-step}: Visualization of a Branch-and-Bound Step.}
This figure illustrates a single iteration of the certified fitting search (\Cref{alg:certified-fit}) applied to a one-dimensional scaling law $L(x; \theta) = E + Ax^{-\alpha}$.
The axes show the parameter coordinates in the reparameterized space defined in \Cref{eq:certified-reparameterization}.
At timestep $t=39$, the box with the lowest loss lower bound is selected for splitting.
The three left columns display two-dimensional orthogonal slices of the parameter-box lattice through the certified (near-)optimal point $(E^\ast, C^\ast, r^\ast)$ in the monotonic coordinates, with background heatmaps showing the per-box loss lower bound $\underline{L}_\texttt{fit}(B)$.
At timestep $t=40$, this box is bisected along the exponent dimension $e^{-\alpha}$, producing two child boxes. The right column shows the corresponding scaling law prediction envelopes for the two child boxes---the fan of curves consistent with each box's parameter range---overlaid on the training data.
Child~1 (lower $e^{-\alpha}$ half) has a loss lower bound exceeding the incumbent's loss and is therefore pruned; Child~2 (containing the certified optimum) is retained.

\paragraph{Figure~\ref{fig:suboptimality-vs-runtime}: Certified Suboptimality vs.\ Runtime.}
For each of the four empirical scaling ladder datasets (\Cref{app:scaling-ladder-datasets}), we compare two fitting methods on the five-parameter Chinchilla scaling law.
The Chinchilla Approach~3 baseline runs L-BFGS sequentially from $K^5$ initializations at grid densities $K \in \{1, 2, \ldots, 8\}$, yielding eight data points per dataset at increasing cumulative wall-clock runtimes (SciPy implementation with the default hyperparameters).
The certified branch-and-bound method (\Cref{alg:certified-fit}), implemented in C++, is run at a sequence of decreasing tolerances $\epsilon \in \{10^{-1}, 3.3 \times 10^{-2}, \ldots, 4 \times 10^{-7}\}$.
For each method and checkpoint, we plot the guaranteed suboptimality interval: the returned candidate's fitting loss minus the tightest certified upper and lower bounds on the global optimum (both obtained from the branch-and-bound search).
For Chinchilla, this interval reflects uncertainty about the true global minimum. Of course, computing these suboptimality bounds requires the certificates provided by our branch-and-bound algorithm.

\paragraph{Figure~\ref{fig:sivia-viable-set}: Sublevel Set Cover and Prediction Envelope.}
This figure demonstrates the certified sublevel set cover algorithm (\Cref{alg:viable-set-cover}) on a synthetic one-dimensional scaling law $L(x;\theta) = E + Ax^{-\alpha}$.
The synthetic dataset consists of 10 observations sampled from a ground-truth scaling law with parameters $(E_\mathrm{true} = 0.5, A_\mathrm{true} \approx 4.58 \times 10^3, \alpha_\mathrm{true} = 0.21)$ over the compute range $x \in [10^{16}, 10^{20}]$ with small multiplicative Gaussian noise ($\sigma = 0.001$).
The fitting loss is the log-space Huber loss of \Cref{sec:chinchilla-fitting} with $\delta = 10^{-3}$. The certified global minimum, with fitting loss $\approx 2.29 \times 10^{-7}$, is first obtained via \Cref{alg:certified-fit} at tolerance $10^{-8}$.
The SIVIA-style cover is then computed with inner and outer fitting-loss tolerances $\epsilon/2 = 10^{-5}$ and $\epsilon = 2\times10^{-5}$, returning $2{,}695$ certified boxes in the reparameterized coordinates $(E,C,r)$ across $17{,}757$ explored nodes in approximately 11\,ms.
The three left panels display two-dimensional projections of the mapped viable set cover $\mathcal C$ in the original coordinates $(E,A,\alpha)$, using $\phi^{-1}$ to convert the search boxes back to these parameters. Light shading shows the full two-dimensional projections; dark shading shows cross-sections with the omitted coordinate fixed at its ground-truth value. Diamonds mark the ground-truth parameters $\theta_{\mathrm{true}}$.
The right panel maps $\mathcal C$ into the prediction space, showing the envelope of scaling curves $L(x;\theta)$ attainable by any $\theta\in\mathcal C$.

\end{document}